\documentclass{article}

\usepackage[final]{neurips_2024}

\usepackage[utf8]{inputenc} 
\usepackage[T1]{fontenc}    
\usepackage{url}            
\usepackage{booktabs}       
\usepackage{amsfonts}       
\usepackage{nicefrac}       
\usepackage{microtype}      

\usepackage{colortbl}
\definecolor{darkblue}{rgb}{0,0,0.5}
\definecolor{turquoise}{RGB}{64, 224, 208}

\usepackage[breaklinks,colorlinks, linkcolor=darkblue, urlcolor=darkblue, citecolor=darkblue]{hyperref}
\usepackage{float}
\usepackage{graphicx}
\usepackage[dvipsnames]{xcolor}
\usepackage{enumitem}
\usepackage{algorithm}
\usepackage{algorithmicx}
\usepackage[indLines=true]{algpseudocodex}
\usepackage{longtable}
\usepackage{wrapfig}
\usepackage{subcaption}
\usepackage{multirow}
\usepackage[export]{adjustbox}
\usepackage{listings}
\usepackage{cprotect}
\usepackage{mathtools}
\usepackage{amsmath}
\usepackage{amsthm}
\usepackage{amssymb}
\usepackage{bbm}

\usepackage{tikz}
\usetikzlibrary{positioning}

\usepackage[capitalize,nameinlink]{cleveref}
\crefname{section}{Sec.}{Secs.}
\crefname{proposition}{Prop.}{Props.}
\crefname{lemma}{Lem.}{Lems.}
\crefname{model}{Mod.}{Mods.}
\crefname{appendix}{App.}{Apps.}
\crefname{theorem}{Thm.}{Thms.}

\newtheorem{theorem}{Theorem}

\definecolor{LightGray}{gray}{0.95}

\DeclarePairedDelimiter\abs{\lvert}{\rvert}%
\DeclarePairedDelimiter\norm{\lVert}{\rVert}%

\makeatletter
\let\oldabs\abs
\def\abs{\@ifstar{\oldabs}{\oldabs*}}

\let\oldnorm\norm
\def\norm{\@ifstar{\oldnorm}{\oldnorm*}}
\makeatother

\usepackage{float}
\floatstyle{ruled}
\newfloat{listing}{htbp}{lst}
\floatname{listing}{Listing}

\lstdefinestyle{custom}{
  breaklines=true,
  breakatwhitespace=true,
  basicstyle=\footnotesize\ttfamily,
  backgroundcolor=\color{LightGray},
  columns=fullflexible,
  showstringspaces=false,
  breakindent=0pt,
  xleftmargin=3pt,
  framesep=3pt,
  frame=leftline,
  rulecolor=\color{LightGray},
  postbreak=\mbox{$\hookrightarrow$\space\space},
}
\title{Recursive Decomposition with Dependencies for Generic Divide-and-Conquer Reasoning}

\author{%
  Sergio Hernández-Gutiérrez\\
  Department of Computer Science\\
  Aalto University\\
  Otakaari 24, 02150 Espoo, Finland\\
  \texttt{sergio.hernandezgutierrez@aalto.fi}\\
  \And
  Minttu Alakuijala\\
  Department of Computer Science\\
  Aalto University\\
  Otakaari 24, 02150 Espoo, Finland\\
  \texttt{minttu.alakuijala@aalto.fi}\\
  \AND
  Alexander V. Nikitin\\
  Department of Computer Science\\
  Aalto University\\
  Otakaari 24, 02150 Espoo, Finland\\
  \texttt{alexander.nikitin@aalto.fi}\\
  \And
  Pekka Marttinen\\
  Department of Computer Science\\
  Aalto University\\
  Otakaari 24, 02150 Espoo, Finland\\
  \texttt{pekka.marttinen@aalto.fi}\\
}

\begin{document}

\maketitle

\begin{abstract}
Reasoning tasks are crucial in many domains, especially in science and engineering. Although large language models (LLMs) have made progress in reasoning tasks using techniques such as chain-of-thought and least-to-most prompting, these approaches still do not effectively scale to complex problems in either their performance or execution time. Moreover, they often require additional supervision for each new task, such as in-context examples. In this work, we introduce Recursive Decomposition with Dependencies (RDD), a scalable divide-and-conquer method for solving reasoning problems that requires less supervision than prior approaches. Our method can be directly applied to a new problem class even in the absence of any task-specific guidance. Furthermore, RDD supports sub-task dependencies, allowing for ordered execution of sub-tasks, as well as an error recovery mechanism that can correct mistakes made in previous steps. We evaluate our approach on two benchmarks with six difficulty levels each and in two in-context settings: one with task-specific examples and one without. Our results demonstrate that RDD outperforms other methods in a compute-matched setting as task complexity increases, while also being more computationally efficient.
\end{abstract}

\section{Introduction}
\label{sec:introduction}
Large language models (LLMs) have been proven successful as the backbone of generic intelligent systems \citep{openaiIntroducingChatGPT, openaiGPT4TechnicalReport2024, gemini2023, anthropicIntroducingNextGeneration2024}. These models possess strong conversational skills \citep{radfordLanguageModelsAre2019, brownLanguageModelsAre2020, chowdheryPaLMScalingLanguage2023, touvronLLaMAOpenEfficient2023}, making them an effective tool to interact with users in a wide range of settings. However, the autoregressive architecture of transformer-based language models limits the complexity of problems that can be solved. As a result, language models struggle with reasoning tasks \citep{nogueiraInvestigatingLimitationsTransformers2021, deletangNeuralNetworksChomsky2022, dziriFaithFateLimits2023a, chenProgramThoughtsPrompting2023a}, from multi-hop question-answering and symbolic manipulation to arithmetic and logical inference. Recent methods have been proposed to increase the performance of LLMs on reasoning problems \citep{weiChainofThoughtPromptingElicits2022,wangSelfConsistencyImprovesChain2022, zhouLeasttoMostPromptingEnables2022, yaoTreeThoughtsDeliberate2023, bestaGraphThoughtsSolving2024,khotDecomposedPromptingModular2022}. In particular, many of these techniques focus on eliciting step-by-step solving processes or decomposition strategies.

Nonetheless, we identify three issues affecting the currently available approaches. First, previous methods typically build a single reasoning chain \citep{weiChainofThoughtPromptingElicits2022,wangSelfConsistencyImprovesChain2022,zhouLeasttoMostPromptingEnables2022,khotDecomposedPromptingModular2022,yaoTreeThoughtsDeliberate2023}, without supporting independent, parallelizable sub-tasks, and allow for limited or no communication between alternative chains. When decomposition is supported, prior work has been limited to fully separable decompositions, without supporting dependencies between sub-tasks \citep{zhang2024examination}. Second, existing methods often require the user to provide task-specific examples \citep{khotDecomposedPromptingModular2022} or a pre-defined decomposition strategy \citep{zhouLeasttoMostPromptingEnables2022,zhang2024examination,bestaGraphThoughtsSolving2024}, making them difficult to incorporate into generic intelligent systems. Third, the number of tokens required to express their reasoning chain frequently scales quadratically with respect to the complexity of the task at hand \citep{zhouLeasttoMostPromptingEnables2022}, which becomes an even more pressing issue when considering the limited context window and high computational cost of LLMs. Additionally, as we empirically show, their downstream performance decays rapidly with increasing task complexity. Our work addresses these issues, empowering LLMs to solve more complex reasoning problems with decomposition strategies that are readily applicable to real-world generic intelligent systems.

We propose Recursive Decomposition with Dependencies (RDD), a flexible and task-agnostic framework for task decomposition with desirable scaling properties and high potential for parallelization. Specifically, we use in-context learning to decompose reasoning problems into sub-problems, solve these individually, and then merge their solutions to solve the original problem. The model can optionally model dependencies between the sub-tasks proposed during the decomposition step.These steps are applied recursively: sub-tasks are repeatedly broken down until either a base case is reached or specific stopping criteria are met. By incorporating relevant in-context examples, sub-task indices, and a scheduler, our method can automatically generate new sub-tasks. It does this by using the results of completed sub-tasks as input for others, allowing the decomposition structure to extend from a simple tree to a more complex directed acyclic graph (DAG). These generic operations (\texttt{split}, \texttt{solve}, and \texttt{merge}) are applicable across tasks and do not require user intervention. Our decomposition strategy reduces the strain on the context window of the model and shortens the runtime per input problem.

Our contributions are:
\begin{enumerate}[topsep=0pt,itemsep=0pt]
    \item \looseness=-1 We introduce a novel method, Recursive Decomposition with Dependencies (RDD), for solving reasoning problems  (\cref{sec:methodology}) via decomposition into smaller subtasks with dependencies,
    \item We empirically demonstrate the effectiveness of our approach, both with and without task-specific demonstrations  (\cref{sec:empirical-evaluation}),
    \item We evaluate our method on one task decomposable into independent sub-problems and another requiring dependency modeling (\cref{sec:empirical-evaluation}).
\end{enumerate}

\section{Recursive Decomposition with Dependencies}
\label{sec:methodology}
\begin{figure*}[!b]
    \centering
    \includegraphics[width=1.0\linewidth]{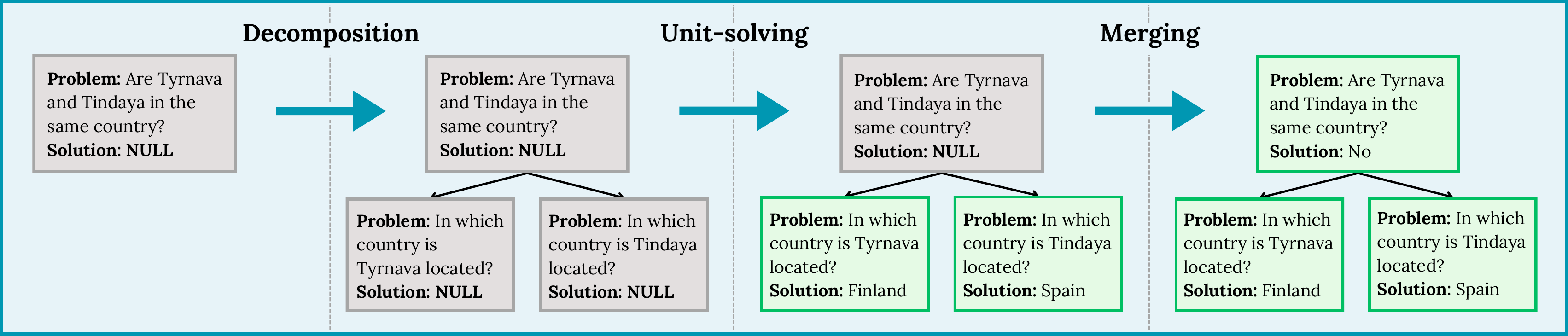}
    \caption{The decomposition methodology pipeline: decomposing, unit-solving, and merging. Nodes in gray represent unsolved problems, while nodes in green represent solved problems.}
    \label{fig:method-pipeline}
\end{figure*}

We assume access to a pretrained LLM and leverage in-context learning and prompting strategies to implement RDD. Our method consists of three steps: \emph{decomposing}, \emph{unit-solving}, and \emph{merging}. We will refer to the initial problem provided by the user as the \textit{root problem} $x_0 \in T^l$, where $T$ is a set of tokens, and $l$ is the length of the prompt.
\textbf{Decomposition:} The root problem is initially decomposed into sub-problems by prompting the LLM with the decomposition meta-task. The model generates either a list of sub-problems or the response ``\textit{This is a unit problem}.'' If the problem is identified to be a unit case, RDD prompts the model to solve it directly. Otherwise, the decomposition meta-task is repeated with each of the sub-problems recursively.
\textbf{Unit-solving:} Unit cases can be solved with either direct input-output prompting, any other existing reasoning method, or by employing an external tool.
\textbf{Merging:} For each set of already-solved sub-problems, we prompt the LLM to merge their solutions to solve their parent problem; we perform this process until we reach the root problem, at which point we obtain the final solution to $x_0$ via the last merging step. A visual representation of this procedure is provided in \cref{fig:method-pipeline} for a non-recursive case of depth one.

\subsection{Notation and definitions}
\label{sec:methodology:notation}

We estimate the conditions under which applying RDD is beneficial compared to a direct solution attempt by estimating the success rates of the decomposition, unit solving, and merging steps. We can define the function predicting the accuracy of the decomposition step as $\phi_{\text{d}, \mathcal{M}_\theta, \mathcal{C}} \left( c, n \right) \in [0, 1]$, and similar functions for the unit-solving and merging steps as $\phi_{\text{u}, \mathcal{M}_\theta, \mathcal{C}} \left( c, n \right) \in [0, 1]$ and $\phi_{\text{m}, \mathcal{M}_\theta, \mathcal{C}} \left( c, n \right) \in [0, 1]$, respectively, where $c$ is the input problem class (e.g., multiplying two numbers), $n$ is the within-class difficulty of the input problem (a problem-specific metric; typically, the size of the input data), $\mathcal{M}_\theta$ is a language model which executes the decomposition, unit-solving and merging steps, and $\mathcal{C}$ is a classifier (in our method, implemented by $\mathcal{M}_\theta$) which returns \texttt{true} if a $(c, n)$ pair constitutes a unit problem and \texttt{false} otherwise. We may refer to the within-class difficulty also as just \textit{difficulty}. Other variables, such as the maximal branching factor or \textit{width} $w$ of each decomposition step, may influence the expected accuracy of RDD, but are treated as fixed constants in this notation. We may also omit $\mathcal{M}_\theta$ and $\mathcal{C}$ for conciseness, assuming them to be constant. We define $\phi_{\text{RDD}}$ to be the overall accuracy of RRD applied with a model $\mathcal{M}_\theta$ and classifier $\mathcal{C}$ on a problem of class $c$ with difficulty $n$. We can also relax our existing notation for all aforementioned functions to only depend on a given random variable $X_0$ from the domain $\mathcal{P}_{c_0, n_0}$, the set of root problem instances $x_0$ belonging to class $c_0$ and with difficulty of $n_0$. We can approximate the individual and overall expected accuracies empirically by measuring their success rates for a large enough set of problem instances, which we explore in \cref{sec:evaluation:error-analysis} within the scope of our evaluation setting.

\paragraph{Transition points} Let $c_0$ be a fixed problem class solvable by $\mathcal{M}_{\theta}$. We expect that, as we lower the within-class difficulty $n_0$ of the problem instances $X_0$ of class $c_0$, we get $\phi_{\text{u}} \left( X_0 \right) \lessapprox 1$. In such cases, we expect $\phi_{\text{RDD}} \left( X_0 \right) \leq \phi_{\text{u}} \left( X_0 \right)$, since decomposing $X_0$ will likely not improve the accuracy of the unit cases sufficiently to compensate for the additional decomposition and merging steps. Thus, we hypothesize the existence of a performance transition point at within-class difficulty $n^{\ast}$, after which $\phi_{\text{RDD}} \left( c_0, n_0 \right) \geq \phi_{\text{u}} \left( c_0, n_0 \right)$ will hold $\forall n_0 \geq n^{\ast}$. We empirically observe such transition points.

\subsection{Methodology}
\label{sec:recursive-decomposition:methodology}

\begin{figure}[ht]
    \centering
    \includegraphics[width=1.0\linewidth,keepaspectratio]{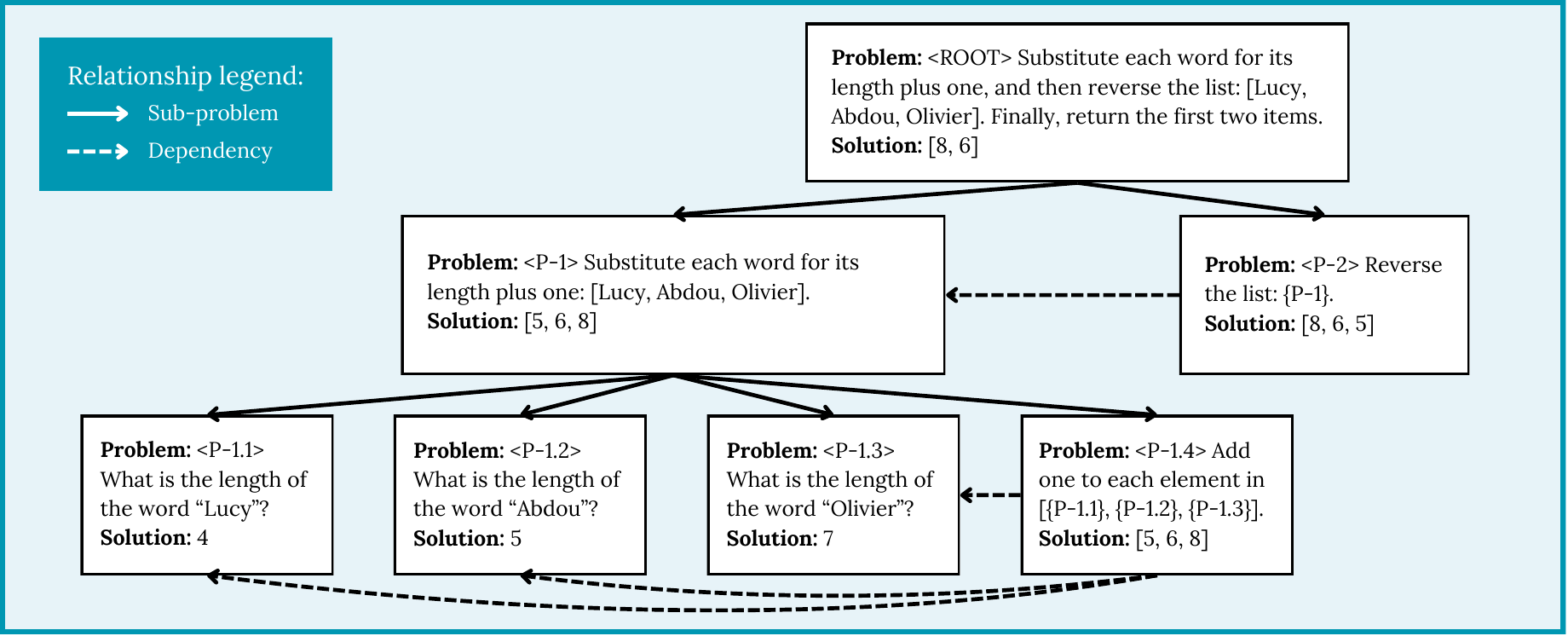}
    \caption{An example of the decomposition graph generated by the RDD method.}
    \label{fig:example-rdd}
\end{figure}

RDD enables the language model $\mathcal{M}_\theta$ to suggest dependencies between sub-problems in the decomposition step.
We request the model to assign a unique identifier (e.g., "\textit{P-1}" and "\textit{P-2}") to each proposed sub-problem. We also encourage the model to cross-reference solutions from other sub-problems via their identifiers (e.g., "\textit{Reverse the following list: \{P-1\}}"). This construction implies that dependent problems cannot be decomposed nor solved without first solving their dependencies. Using the dependencies specification generated by the model, we can now define edges between sub-problems with a common parent in the decomposition tree, resulting in a directed acyclic graph (DAG). \cref{fig:example-rdd} shows an example decomposition graph produced by RDD.

\paragraph{Information flow}

When solving sub-problems, we aim to minimize the amount of information about ancestor tasks included in the context to isolate the relevant information and achieve better scaling properties with respect to the difficulty of the root problem. When prompting for a decomposition, only the current problem description is provided to the model, along with a description and a set of demonstrations of the meta-task (e.g., decomposition or merging). We do not include the history of ancestor problem descriptions, which increases in size with the depth of the recursion process; this feature requires a language model to always provide all needed data and instructions in the description of each sub-problem. In the merging step, the decomposition of the current level and its sub-solutions are provided.

\paragraph{Maximizing generic applicability}

We consider a fixed set of generic meta-task demonstrations included in the decomposition, merging, and unit-case prompts. This set exhibits a diverse range of tasks. We show experimentally that this same set of generic examples is effective for guiding decomposition for new, unseen tasks. The examples we use in our evaluation are available in \cref{sec:appendix:examples}. Moreover, the meta-tasks RDD performs are fixed, regardless of the input problem, and thus also task-invariant. The generality of the demonstrations is a spectrum and thus presents a trade-off between the degree of applicability of the methodology and the degree of assistance it provides to $\mathcal{M}_\theta$, the latter variable being correlated with performance. To increase performance in a given domain (e.g., programming assistance) at the cost of generality, the generic demonstrations can be selected from the same domain (e.g., coding problems).

\begin{wrapfigure}{R}{0.55\textwidth}
    \vspace{-5ex}
    \begin{minipage}{0.55\textwidth}
    
        \begin{algorithm}[H]
        \footnotesize
        \caption{\textsc{ScheduleBFS}}
        \label{alg:solve-bfs}
        \textbf{Input:} problem
        \begin{algorithmic}[1]
            \State unsolved $\gets$ empty queue
        
            \For{dependency $\in$ problem.dependencies}
                \State \Call{ScheduleBFS}{dependency}
            \EndFor
            
            \State sub-problems $\gets$ \Call{Decompose}{problem}
            \For{sub-problem $\in$ sub-problems}
                \State Add sub-problem to unsolved
            \EndFor
        
            \While{unsolved \textbf{is not} empty}
                \State next-problem $\gets$ unsolved.front
        
                \For{dependency $\in$ next-problem.dependencies}
                    \State \Call{ScheduleBFS}{dependency}
                \EndFor
                
                \State sub-problems $\gets$ \Call{Decompose}{next-problem}
                \For{sub-problem $\in$ sub-problems}
                    \State Add sub-problem to unsolved
                \EndFor
            \EndWhile
        
            \State \Return \Call{ScheduleDFS}{problem, [ ]}
        \end{algorithmic}
        \end{algorithm}

    \end{minipage}
    \vspace{-2ex}
\end{wrapfigure}

\paragraph{Scheduler} The scheduler defines the execution order of the decomposition, unit-solving, and merging steps. The order of decomposition defines the structure of the resulting graph. The root problem is expanded via breadth-first search (BFS) traversal until an unsolved dependency is found, in which case the current traversal process halts and executes the BFS routine with the dependency as the root problem. A depth-first search (DFS) traversal schedules the unit-solving and merging steps. The complete procedure we employ is explicitly reflected in \cref{alg:solve-bfs}. The \textsc{ScheduleDFS} procedure is provided in \cref{sec:appendix:scheduler}; the \textsc{Decompose} routine corresponds to the decomposition step. An example execution of this algorithm shown in \cref{fig:example-rdd} is provided in \cref{sec:appendix:example-run}. For a parallelized implementation, the scheduler synchronizes the execution of sub-problems with inter-dependencies.

\section{Empirical Evaluation}
\label{sec:empirical-evaluation}
The hypotheses we aim to validate through our empirical study are the following:

\begin{itemize}[topsep=0pt,itemsep=0pt]
    \item \textbf{Hypothesis 1}: RDD increases accuracy in complex reasoning problems over state-of-the-art methods in a compute-matched setting.
    \item \textbf{Hypothesis 2}: The recursive decomposition technique augments the model's reasoning abilities even in the absence of task-specific data.
    \item \textbf{Hypothesis 3}: Solving reasoning problems via RDD reduces the time taken to reach a solution compared to solving the entire problem via step-by-step prompting strategies.
    \item \textbf{Hypothesis 4}: RDD reduces the average amount of tokens per generation process, thus lessening the strain on the context window.
\end{itemize}

As baselines, we consider Chain-of-Thought (CoT; \citet{weiChainofThoughtPromptingElicits2022}) and Least-to-Most prompting (LtM; \citet{zhouLeasttoMostPromptingEnables2022}). We use self-consistency (SC; \citep{wangSelfConsistencyImprovesChain2022}) to align the amount of computation between our method and the baselines. In our implementation of SC, we employ the LLM itself to decide the most consistent answer given the set of sampled solutions. The first solution candidate is sampled greedily, while the rest are sampled with a temperature of 0.7 to produce a variety of reasoning chains. Given that each SC sample will produce a large number of generated tokens and including them all in a single context window can be challenging, we propose to use binary search to find the most consistent answer. RDD employs a single CoT or LtM chain at the unit-solving prompt but does not use self-consistency to aggregate multiple answers. We evaluate all methods on benchmark tasks of increasing difficulty. For each difficulty, scores are averaged across the same 100 randomly sampled problem instances for all methods. We employed the instruction-tuned Llama 3 70B \citep{IntroducingMetaLlama} as the underlying model. This model was run on NVIDIA A100 and H100 GPUs. \cref{sec:appendix:resource-usage} provides resource usage statistics for all experiments in this section. It is additionally possible to parallelize the solving of independent sub-problems for a speedup.

\subsection{Task-specific Experiments}
\label{sec:evaluation:task-specific}

\begin{wrapfigure}{R}{0.5\textwidth}
    \centering
    \vspace{-1.4cm}
\includegraphics[width=1.0\linewidth]{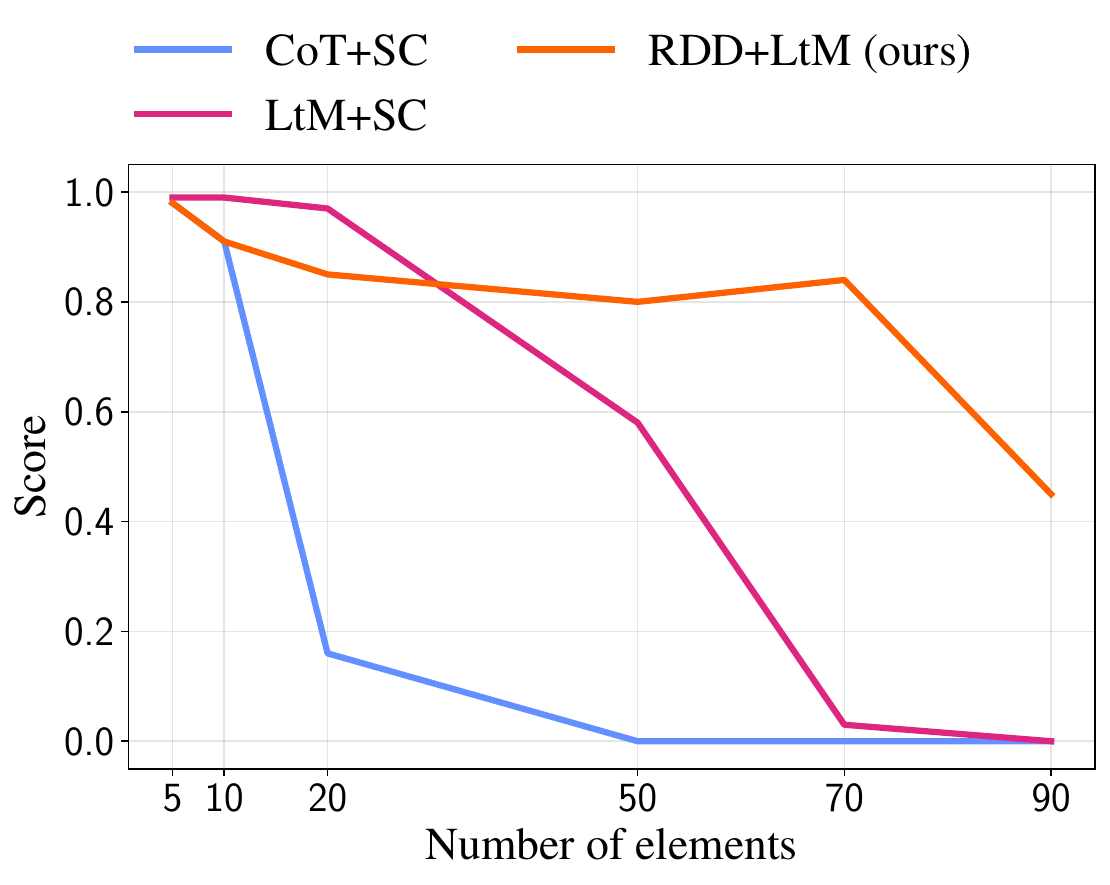}
    \caption{An evaluation of RDD against CoT \citep{weiChainofThoughtPromptingElicits2022} and LtM \citep{zhouLeasttoMostPromptingEnables2022} with self-consistency (SC; \citet{wangSelfConsistencyImprovesChain2022}) on the letter concatenation benchmark in the task-specific few-shot setting. Our system uses LtM at the unit-solving step; we refer to it as RDD+LtM.}
    \label{fig:evaluation:specific-letterconcat}
    \vspace{-.5cm}
\end{wrapfigure}

We first experiment with task-specific examples to validate our approach. Each call to the model (i.e., all baseline calls, as well as the decomposition, unit-solving, and merging steps) operates in a 5-shot in-context setting. These examples are in-distribution with respect to the problem class $c_0$, but out-of-distribution with respect to the within-class difficulty $n_0$. For this experiment, we evaluate on the letter concatenation problem, which asks the LLM to concatenate the $i$'th character of each word in a list. This task can be recursively decomposed into independent sub-problems, thus not requiring dependency modeling. The difficulty $n_0$ is the number of words in the list.

\cref{fig:evaluation:specific-letterconcat} demonstrates the results for six input sizes. The score is computed as the average of an exact match metric. We find RDD to outperform the baselines as the task complexity increases, confirming Hypothesis 1. For $n_0 < 20$, it does not seem beneficial to recursively decompose the problem. We observe a transition point $20 < n^{\ast}_0 < 50$ with respect to LtM+SC. In \cref{tab:resource-usage:letterconcat-specific} (\cref{sec:appendix:resource-usage}), we show that RDD also reduces execution time with respect to the baselines.

\subsection{Generic Experiments}
\label{sec:evaluation:generic}

\begin{figure}[ht]
    \centering
    \captionsetup[subfigure]{justification=centering}
    \begin{subfigure}[t]{0.49\linewidth}
        \centering
        \includegraphics[width=1.0\linewidth]{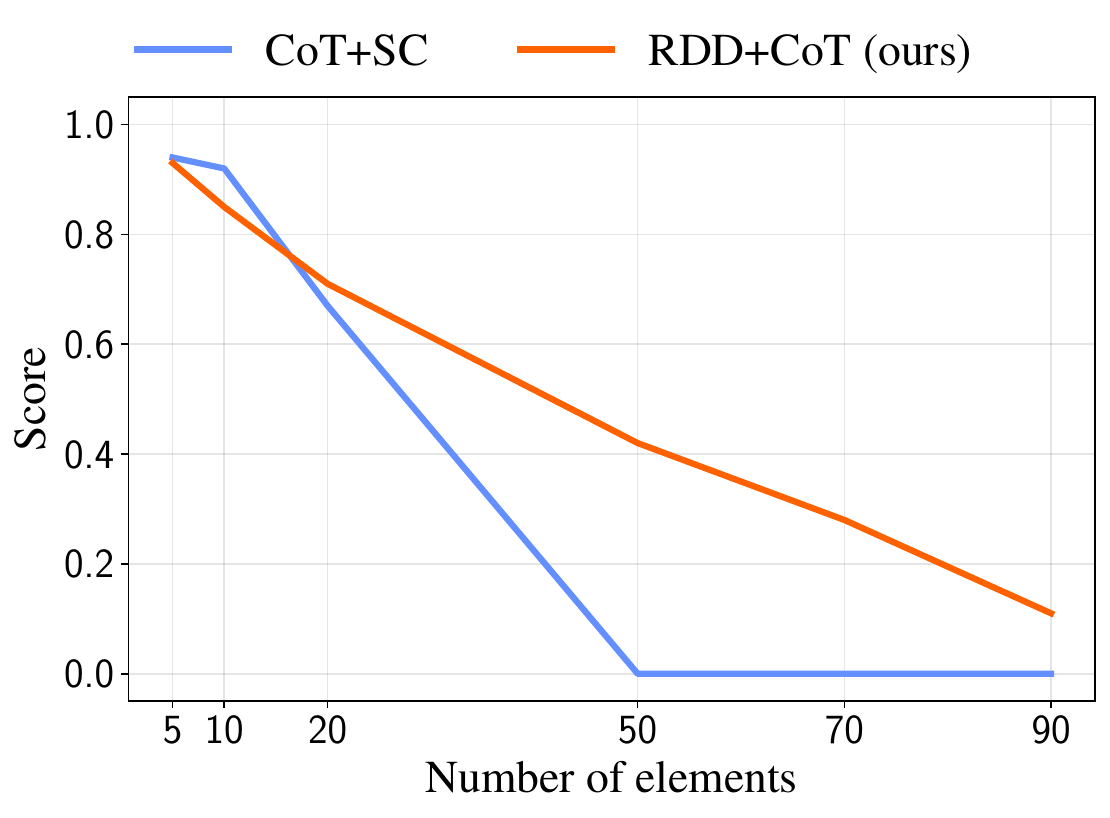}
        \caption{Letter Concatenation}
        \label{fig:evaluation:generic-letterconcat}
    \end{subfigure}
    \hfill
    \begin{subfigure}[t]{0.49\linewidth}
        \centering
            \includegraphics[width=1.0\linewidth]{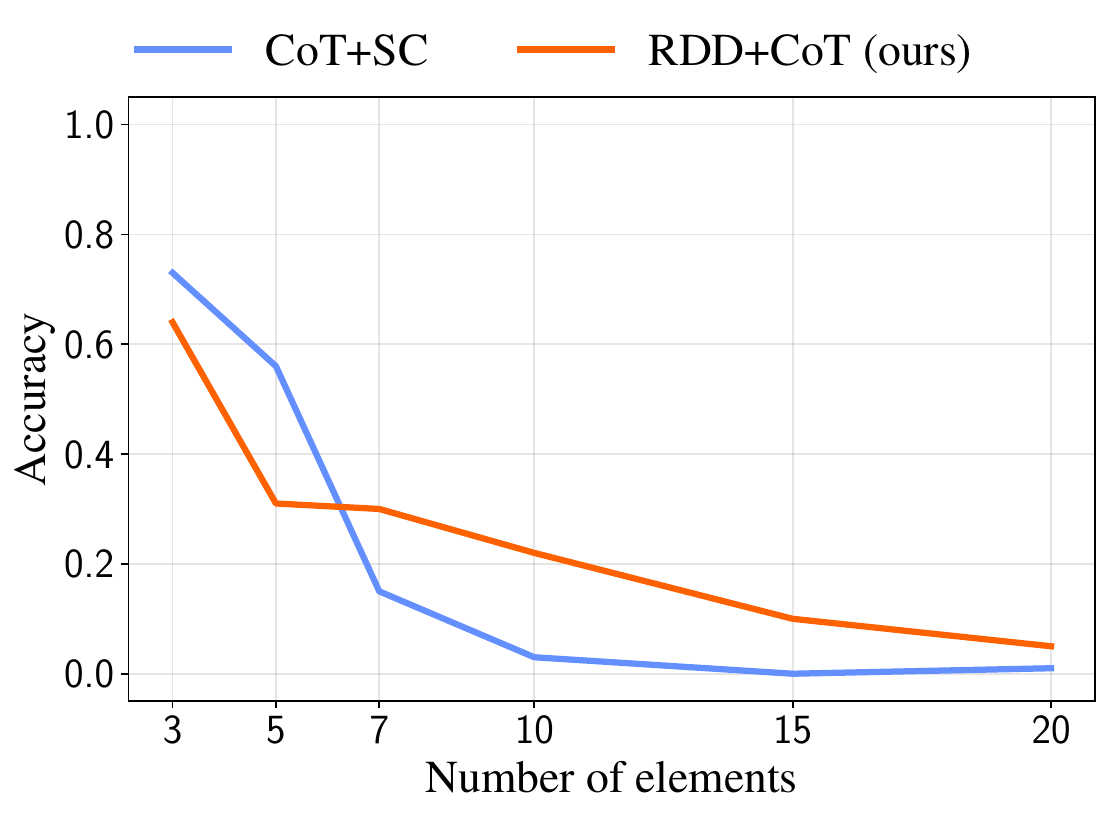}
    \caption{Length Reversal}
    \label{fig:evaluation:deps-lengthrev}
    \end{subfigure}

    \caption{An evaluation of RDD against CoT with self-consistency (SC) the generic few-shot setting. Our system uses CoT at the unit-solving step; we refer to it as RDD+CoT.}
    \label{fig:evaluation:generic}
\end{figure}

To evaluate Hypothesis 2, we verify whether the advantage of RDD is maintained when task-specific in-context examples are replaced with generic examples. These generic examples depict a wide range of tasks: arithmetic, coding, symbolic manipulation, multi-step logic reasoning, and multi-hop QA, but exclude the tested problem class $c_0$. In this experiment, we operate in a 5-shot setting for the unit-solving and merging steps, and a 7-shot setting for the decomposition step. We provide the included examples in \cref{sec:appendix:examples}. We also improved the description of the problem with respect to the one used for the experiments described in \cref{sec:evaluation:task-specific}: the model is tasked with concatenating each character using a space as a delimiter. If the characters are concatenated without a special separator, their tokenization may change when encoding them for subsequent steps in the autoregressive generation process. This behavior has been also described by \citet{mirchandaniLargeLanguageModels2023} for tasks represented in grids of numbers. The rest of the setup is the same as previously described.

For this set of experiments, we compare RDD with CoT as the unit-solving method (RDD+CoT) against CoT with self-consistency (CoT+SC). In \cref{fig:evaluation:generic-letterconcat}, we can observe that RDD again outperforms this baseline as the difficulty of the task increases. We see a performance transition point $10 < n^{\ast}_0 < 20$ with respect to CoT+SC. We again observe considerable time savings; a complete account of resource usage can be found in \cref{tab:resource-usage:letterconcat-generic} (\cref{sec:appendix:resource-usage}).

\subsection{Sub-tasks with Dependencies}
\label{sec:modeling-dependencies}

This section shows the results of our evaluation with RDD on the task of length reversal. To solve this task, the model must substitute each word in a list with its length (number of characters), and then reverse the order of the items in the list. This task benefits from dependency modeling in the decomposition step. We compare our method against CoT+SC with generic in-context examples. We use five examples for the unit-solving and merging steps and eight for the decomposition step. The examples showcase several different decomposition and merging patterns and also cover a wide range of problem classes as previously described. We also augment the meta-prompt for the decomposition step to include instructions describing how to suggest dependencies (see \cref{sec:appendix:prompts}).

\cref{fig:evaluation:deps-lengthrev} demonstrates the results of our experiment on the length reversal benchmark. We observe a transition point $5 < n^{\ast} < 7$, after which RDD outperforms CoT. \cref{tab:resource-usage:lengthrev-deps-generic} demonstrates that the time taken by RDD to complete the experiment is considerably lower than that of the CoT method.

\subsection{Error analysis}
\label{sec:evaluation:error-analysis}

\begin{figure}[ht]
    \centering
    \captionsetup[subfigure]{justification=centering}
    \begin{subfigure}[t]{0.49\linewidth}
        \centering
        \includegraphics[width=1.0\linewidth]{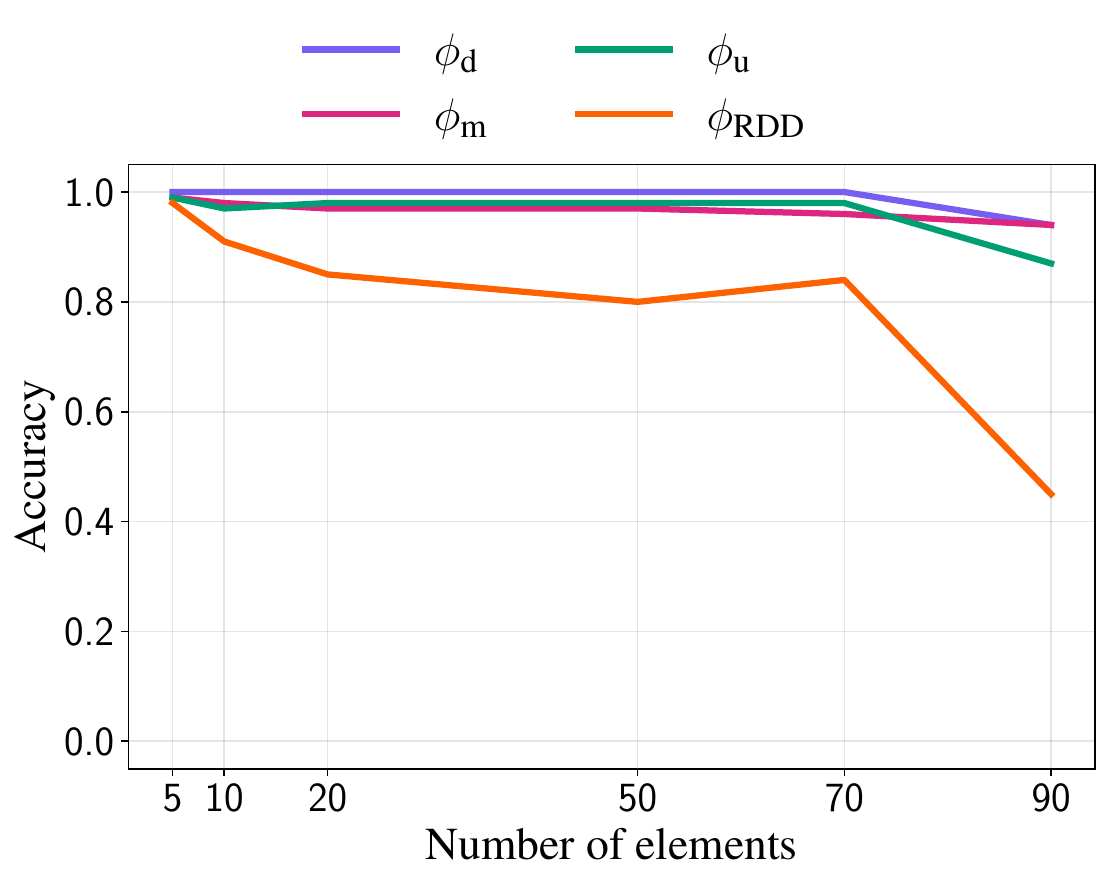}
        \caption{Error analysis for the task-specific setting.}
        \label{fig:evaluation:errors:letterconcat-specific}
    \end{subfigure}
    \hfill
    \begin{subfigure}[t]{0.49\linewidth}
        \centering
        \includegraphics[width=1.0\linewidth]{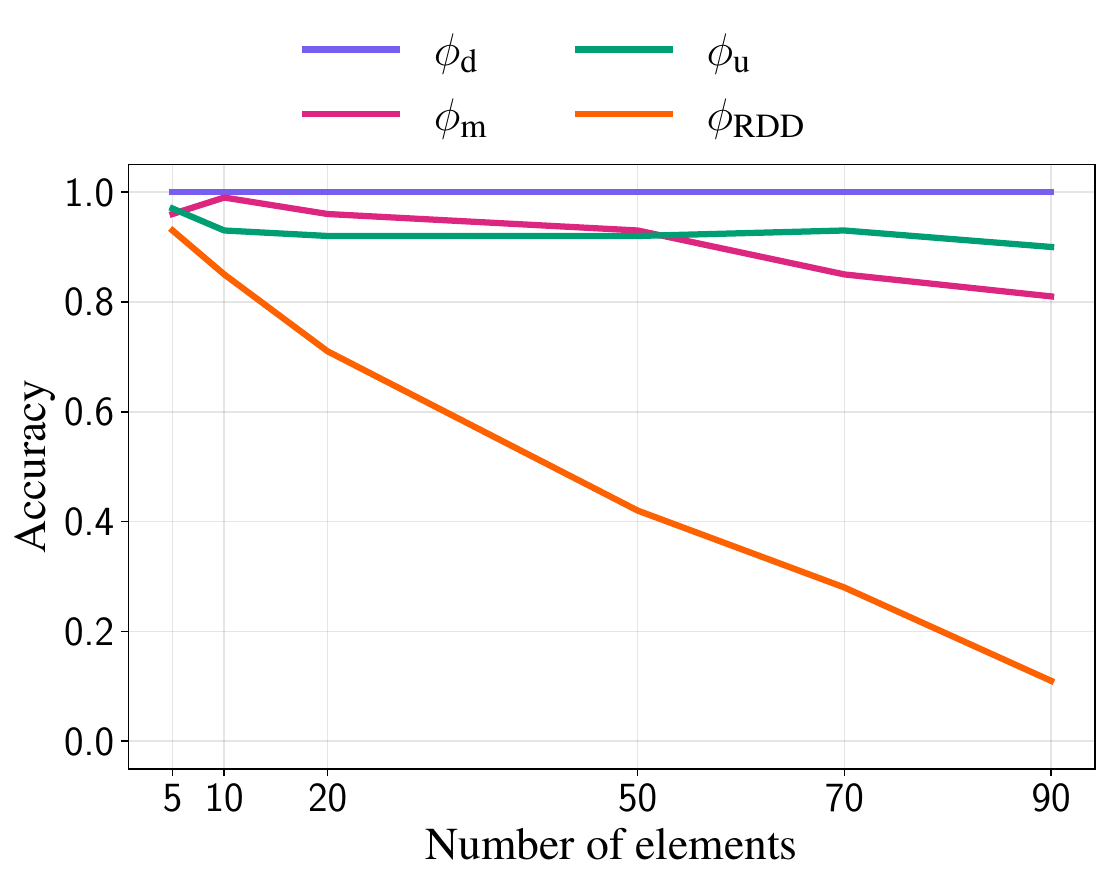}
        \caption{Error analysis for the generic setting.}
        \label{fig:evaluation:errors:letterconcat-generic}
    \end{subfigure}

    \caption{The error sources of the recursive decomposition approach in the letter concatenation benchmark with respect to $n_0$ (the size of the list in the problem) for task-specific in-context \textbf{(a)} and generic \textbf{(b)} experiments . $\phi_{\text{d}}$ corresponds to the observed success rate in the decomposition step, $\phi_{\text{m}}$ in the merging step and $\phi_{\text{u}}$ in the unit-case. The values are computed using all problem classes $c_i$ and within-class difficulties $n_i$ appearing in the decomposition graph. $\phi_{\text{RDD}}$ is the end-to-end accuracy.}
    \label{fig:evaluation:errors}
\end{figure}

\paragraph{Error sources} To analyze the sources of errors when employing RDD, we empirically quantify the accuracies $\phi_{\text{d}}$, $\phi_{\text{u}}$ and $\phi_{\text{m}}$ of the individual steps. Using these, we estimate the overall accuracy of the system $\phi_{\text{RDD}}$, as described in \cref{sec:methodology:notation}. We perform this analysis using the data from our experiments on the letter concatenation task. The results are shown in \cref{fig:evaluation:errors:letterconcat-specific} for the task-specific in-context setting and \cref{fig:evaluation:errors:letterconcat-generic} for the generic in-context one, with the exact numbers included in \cref{sec:appendix:error-analysis-data}. To compute these statistics, we identified the accuracies of the decomposition, unit solving, and merging steps. For decomposition and merging steps, we consider whether the problem at hand was decomposed or merged correctly. We employ an exact-match metric for all measured accuracies and average them for all instances of each task. Note that these values are averaged for all problem classes $c_i$ and within-class difficulties $n_i$ which appear recursively in the solving process; importantly, $\phi_{\text{u}}$ does not correspond to $\phi_{\text{u}} \left( c_0, n_0 \right)$. In \cref{fig:error-propagation} (for more detail, see \cref{sec:appendix:error-propagation}), we provide an example of an error in the unit case which was common in our evaluation; errors made when solving a sub-problem often carry over to their parent problems during the merging step, which eventually can produce an erroneous final solution of the root problem $x_0$.

\paragraph{Error recovery}

Our framework is capable of recovering from the errors from the individual steps. To elicit this behavior, we include the sequence "\textit{If you find any mistakes in the sub-solutions, you can fix the mistakes while you merge the sub-solutions}" in the merge step prompt. This is a key mechanism when modeling sub-problem dependencies: if the model does not follow the syntax to specify dependencies correctly, these may not be recognized by our parser. For instance, this may result in a sub-problem statement with missing data such as "\textit{Reverse the following list: }". The description-solution pair of this sub-problem will be straightforward to recognize as a mistake. Since we provide the model with the top problem description as well as the sub-problems' descriptions in the merging step, it can identify such issues and propose a merged solution correcting the erroneous sub-solutions. We have observed that, in these cases, the model simply regards the top problem as a unit case and attempts to solve it directly; if it succeeds, we deem this behavior as error recovery. In \cref{fig:error-recovery} (\cref{sec:appendix:error-recovery}), we can observe an example of the merging step in the root problem recovering from an error made when solving its sub-problems. The meaning of the merging step changes if an error recovery behavior is possible: the merging step becomes a special type of unit-solving the root problem with additional context (which may or may not be informative), and it is no longer dependent on the accuracy of unit-solving the sub-problems.

\subsection{Space and time efficiency}
\paragraph{Execution time} Hypothesis 3 has also been empirically proven by our results. In the tables provided in \cref{sec:appendix:resource-usage}, we can observe that the time RDD takes to complete experiments is lower than the time the baselines take. Note that these values include all samples and voting calls required by self-consistency; using vanilla CoT or LtM would be faster, but we strove to compare our method to baselines with access to similar computational resources. We hypothesize that higher time efficiency is achieved due to fewer output tokens generated by our implementation. Given the quadratic space complexity of the baseline methods with respect to the difficulty of the problems, we expect that recursively decomposing the root problem will lower the number of tokens generated by the model required to reach a solution. Since each output token is generated via a full forward pass of the underlying network, it is expected that a lower amount of forward passes will result in proportional time savings. Implementing the parallelization of independent steps in RDD would further increase time savings.

\looseness=-1
\paragraph{Reduced context length} By dividing the number of context and output tokens by the number of calls, we can see that Hypothesis 4 is confirmed: recursive decomposition helps alleviate issues relating to the overflow of the context window as the complexity of the tasks increases. We hypothesize this number is lower for RDD because of the space scaling properties of CoT-based methods.

\section{Related Work}
\label{sec:related-works}
\paragraph{Expressive power of the reasoning graph}
Methods based on step-by-step decomposition, such as Chain-of-Thought \citep{weiChainofThoughtPromptingElicits2022}, Socratic CoT \citep{shridharDistillingReasoningCapabilities2023}, Least-to-Most prompting \citep{zhouLeasttoMostPromptingEnables2022}, Plan-and-Solve prompting \citep{wangPlanandSolvePromptingImproving2023a}, iterative prompting \citep{wangIterativelyPromptPretrained2022a} PAL \citep{gaoPALProgramaidedLanguage2023a}, Parsel \citep{zelikmanParselAlgorithmicReasoning2023} or the method proposed by \citet{perezUnsupervisedQuestionDecomposition2020a} to decompose multi-hop QA tasks, can be understood as chain-like decompositions of a problem. Tree-of-Thoughts (ToT; \citep{yaoTreeThoughtsDeliberate2023}) builds a tree; however, the structure of this graph represents a sampling process, not a recursive decomposition. \citet{zhang2024examination} propose a tree-like recursive decomposition strategy, not considering sub-problem dependencies. Although DecomP \citep{khotDecomposedPromptingModular2022} performs steps in sequence in a chain-like fashion, it also implicitly models the solving process as a directed acyclic graph (DAG) via tool usage, but its structure needs to be demonstrated by the user for every problem class; this graph is also not modeling dependencies between sub-problems. Graph of Thoughts \citep{bestaGraphThoughtsSolving2024} explicitly uses a DAG, but both its structure and the meaning of the nodes (i.e., sub-problem descriptions) must be provided by the user for every problem instance. Instead, RDD enables the model to explicitly model dependencies without user input beyond the initial problem description, resulting in a DAG.

\paragraph{Generic applicability} The previously mentioned methods modeling the reasoning process as a chain are often able to demonstrate their decomposition strategies with generic in-context examples (e.g., few-shot CoT by \citet{brownLanguageModelsAre2020}), that is, without the user needing to provide examples for each new problem instance. However, some of these strategies cannot be applied to any arbitrary problem class, such as LtM \citep{libbyComparisonMosttoLeastLeasttoMost2008, zhouLeasttoMostPromptingEnables2022}. More complex methods \citep{yaoTreeThoughtsDeliberate2023, khotDecomposedPromptingModular2022, bestaGraphThoughtsSolving2024} require extensive user-generated input to model the reasoning process. In contrast, RDD can model complex reasoning structures without unrealistic data requirements at runtime.

\paragraph{Parallelization of problem-solving process}
Similar to Skeleton-of-Thought \citep{ningSkeletonofThoughtPromptingLLMs2023}, we enable parallel decoding of the solution to a problem by identifying independent steps in the reasoning that can be computed in parallel.
However, SoT does not decompose reasoning chains; the authors state that it is challenging to apply their method on problems that require step-by-step thinking. Other aforementioned methods relying on sequential decomposition do not allow for the parallelization of reasoning steps.

\section{Conclusion}
\label{sec:conclusion}
In this work, we have developed a recursive decomposition technique for LLMs allowing for sub-problem dependencies. Our empirical evaluation considered two benchmarks on six levels of increasing difficulty and two settings of varying degrees of task-specific resource availability. Based on our experiments, we also analyzed the nature of the errors our method makes during the solving process. RDD outperforms state-of-the-art baselines as the difficulty of the tasks increases. Moreover, RDD is allows for ordered execution of subtasks and their dependencies, parallelizable by design, allows for error recovery, and achieves lower time and space complexities than the existing baselines.

Our results show that recursively decomposing reasoning problems with general-purpose LLMs is feasible, and can provide significant performance and resource usage benefits for complex tasks. The design of the RDD methodology and its demonstrated viability without task-specific support remove integration barriers towards existing LLM-based systems. We believe that our proposed method and our findings can be used to advance the reasoning capabilities of real-world language-based AI systems to improve type 2 reasoning capabilities. Future work may explore alternative implementations of the unit-problem classifier, quantify the speedup achieved by a parallelized implementation of RDD, and develop improved strategies for decomposition.

\textbf{Limitations}
It is important to acknowledge that other methods, such as CoT, can outperform our methods for easier problems (i.e., before the transition point). Additionally, some problems, especially in natural sciences, are not readily decomposable and might be challenging for RDD.

\bibliography{main}
\bibliographystyle{iclr2025_conference}

\clearpage
\appendix

\clearpage
\section{Requirements for improved efficacy}
\label{sec:appendix:desideratum}
We can recursively formulate the expected accuracy of our decomposition method as
\begin{equation}
    \phi_{\text{RDD}} \left( X_0 \right) = \phi_{\text{d}} \left( X_0 \right) \phi_{\text{m}} \left( X_0 \right) \prod^w_{i=1} \left[ \mathbbm{1} \left[ \mathcal{C} \left( X_i \right) \right] \phi_{\text{u}} \left( X_i \right) + \mathbbm{1} \left[ \neg \mathcal{C} \left( X_i \right) \right] \phi_{\text{RDD}} \left( X_i \right) \right].
    \label{eq:recursive-decomposition:accuracy}
\end{equation}
Each random variable $X_i$ has as domain the set of sub-problems $x_i$ resulting from the decomposition of $X_0$ performed by $\mathcal{M}_{\theta}$, and we can attribute a class $c_i$ and difficulty $n_i$ to each of them. We also assume a constant width value $w$, corresponding to the number of sub-problems that a decomposition always produces.

We can identify several requirements for the following desideratum to hold:
\begin{equation}
    \phi_{\text{RDD}} \left( X_0 \right) > \phi_{\text{u}} \left( X_0 \right).
    \label{eq:recursive-decomposition:desideratum}
\end{equation}
Our desideratum states that we are more likely to arrive at a correct solution for a problem instance $x_0$ if we recursively decompose it instead of solving it directly as a unit case.

\begin{theorem}[Decomposition and merging requirement]
In order for the desideratum in \cref{eq:recursive-decomposition:desideratum} to hold, it is required that
\begin{equation}
    \phi_{\text{d}} \left( X_0 \right) \phi_{\text{m}} \left( X_0 \right) > \phi_{\text{u}} \left( X_0 \right).
    \label{eq:recursive-decomposition:condition1}
\end{equation}
\end{theorem}
In other words, the probability of decomposing the problem into sub-problems and then merging their sub-solutions should be greater than the probability of solving the problem directly, so that the additional non-zero probability of obtaining wrong solutions for each of the sub-problems is balanced out. In simpler terms, the tasks of decomposing and merging a problem instance $x_0$ must be easier than the task of solving $x_0$ without decomposition.

\begin{proof}
We can prove this requirement by contradiction. Let us assume that our desideratum stated in \cref{eq:recursive-decomposition:desideratum} can hold when $\phi_{\text{d}} \left( X_0 \right) \phi_{\text{m}} \left( X_0 \right) \leq \phi_{\text{u}} \left( X_0 \right)$. Without loss of generality, we can use the notation $\phi_{\text{u/RDD}} \left( X_i \right)$ to refer to the accuracies of the unit and non-unit cases for the sub-problems $X_i$ of $X_0$ indifferently, $\forall i \in [1, w]$. We can employ the definition of $\phi_{\text{RDD}}$ given in \cref{eq:recursive-decomposition:accuracy} as
\begin{equation}
    \phi_{\text{RDD}} \left( X_0 \right) = \phi_{\text{d}} \left( X_0 \right) \phi_{\text{m}} \left( X_0 \right) \phi_{\text{u/RDD}} \left( X_1 \right) \dots \phi_{\text{u/RDD}} \left( X_w \right).
\end{equation}
Since all accuracies are in the range $[0, 1]$, this leads to
\begin{equation}
    \phi_{\text{RDD}} \left( X_0 \right) \leq \phi_{\text{d}} \left( X_0 \right) \phi_{\text{m}} \left( X_0 \right).
\end{equation}
Given our initial assumption, we can conclude that
\begin{equation}
    \phi_{\text{RDD}} \left( X_0 \right) \leq \phi_{\text{u}} \left( X_0 \right),
\end{equation}
which is a contradiction, as we stated that our desideratum would hold.
\end{proof}

\begin{theorem}[Unit case requirement]
An additional requirement for the desideratum in \cref{eq:recursive-decomposition:desideratum} to hold is that
\begin{equation}
    \phi_{\text{u}} \left( X_i \right) > \phi_{\text{u}} \left( X_0 \right) , \forall i \in [1, w].
    \label{eq:recursive-decomposition:condition2}
\end{equation}
\end{theorem}

\begin{proof}
We can prove that this requirement is needed by contradiction. Assume $\phi_{\text{u}} \left( X_i \right) \leq \phi_{\text{u}} \left( X_0 \right) , \exists i \in [1, w]$, such that \cref{eq:recursive-decomposition:desideratum} holds. Let us first consider the case when $\mathcal{C}$ classifies $X_i$ as a unit problem. All other sub-problems $X_j, \text{ s.t. } j \in [1, w] \setminus \{i\}$, may be classified as either unit or non-unit problems without loss of generality. We can use \cref{eq:recursive-decomposition:accuracy} to formulate
\begin{equation}
    \phi_{\text{RDD}} \left( X_0 \right) = \phi_{\text{d}} \left( X_0 \right) \phi_{\text{m}} \left( X_0 \right) \phi_{\text{u/RDD}} \left( X_1 \right) \dots \phi_{\text{u}} \left( X_i \right) \dots \phi_{\text{u/RDD}} \left( X_w \right).
\end{equation}
By definition, all accuracies are in the range $[0, 1]$. Hence, we have that
\begin{equation}
    \phi_{\text{RDD}} \left( X_0 \right) \leq \phi_{\text{u}} \left( X_i \right).
    \label{eq:recursive-decomposition:unit-case-heuristic:statement1}
\end{equation}
Given our initial assumption, we can state that
\begin{equation}
    \phi_{\text{RDD}} \left( X_0 \right) \leq \phi_{\text{u}} \left( X_0 \right),
    \label{eq:recursive-decomposition:unit-case-heuristic:statement2}
\end{equation}
which is a contradiction, as our initial claim was that $\phi_{\text{RDD}} \left( X_0 \right) > \phi_{\text{u}} \left( X_0 \right)$. We can see that this proof can be easily extended to the case where there exists more than one sub-problem $X_i$ such that $\phi_{\text{u}} \left( X_i \right) \leq \phi_{\text{u}} \left( X_0 \right)$. Let us now consider the second case of $\mathcal{C}$ classifying $X_0$ as a non-unit problem. In this case, we can construct a similar proof for this recursive scenario. Similar to the previous case, we can use \cref{eq:recursive-decomposition:accuracy} as
\begin{equation}
    \phi_{\text{RDD}} \left( X_0 \right) = \phi_{\text{d}} \left( X_0 \right) \phi_{\text{m}} \left( X_0 \right) \phi_{\text{u/RDD}} \left( X_1 \right) \dots \phi_{\text{RDD}} \left( X_i \right) \dots \phi_{\text{u/RDD}} \left( X_w \right).
\end{equation}
Since all accuracies are in the range $[0, 1]$, we have that
\begin{equation}
    \phi_{\text{RDD}} \left( X_0 \right) \leq \phi_{\text{RDD}} \left( X_i \right).
\end{equation}
If we consider $X_i$ as the root problem of its own decomposition sub-graph, we can use the same derivation process leading to \cref{eq:recursive-decomposition:unit-case-heuristic:statement2} to state that
\begin{equation}
    \phi_{\text{RDD}} \left( X_i \right) \leq \phi_{\text{u}} \left( X_i \right),
\end{equation}
and thus
\begin{equation}
    \phi_{\text{RDD}} \left( X_0 \right) \leq \phi_{\text{u}} \left( X_i \right).
\end{equation}
We have reached the same statement described in \cref{eq:recursive-decomposition:unit-case-heuristic:statement1}, which we have already proven to lead to a contradiction. If we continue decomposing the sub-problems recursively, we can keep extending our proof until a unit case is reached (which can be guaranteed given strict termination criteria).
\end{proof}

\clearpage
\section{Scheduler algorithm}
\label{sec:appendix:scheduler}
 The procedure for the DFS scheduler is described in \cref{alg:solve-dfs}. The \textsc{Decompose} procedure corresponds to the decomposition step and the \textsc{MergeOrUnit} procedure to either the unit-solving or merging steps performed by $\mathcal{M}_\theta$.

\begin{algorithm}[H]
    \caption{\textsc{ScheduleDFS}}
    \label{alg:solve-dfs}
    \textbf{Input:} problem, visited
    \begin{algorithmic}[1]
        \If{problem $\in$ visited}
            \State \textbf{raise} a cycle error
        \EndIf
        \State visited $\gets$ visited $\cup$ problem

        \If{problem \textbf{is} decomposed}
            \State sub-problems $\gets$ problem.sub-problems
        \Else
            \State sub-problems $\gets$ \Call{Decompose}{problem}
        \EndIf
    
        \For{sub-problem $\in$ sub-problems}
            \State \Call{ScheduleDFS}{sub-problem, visited}
        \EndFor
        \State \Return \Call{MergeOrUnit}{problem}
    \end{algorithmic}
\end{algorithm}

\clearpage
\section{Prompts}
\label{sec:appendix:prompts}
\begin{listing}[H]
\caption{CoT prompt, for both the baseline and unit cases.}
\begin{lstlisting}
Your task is to solve the problem below. You can reason about the problem before stating your answer. The answer MUST be between the following tags: <ANSWER>...</ANSWER>. An example is provided to showcase how to use the tags; you must only solve the last problem given.

## Examples

{examples}

## Problem

Problem: {problem}
Answer: Let's think step by step.
\end{lstlisting}
\end{listing}

\begin{listing}[H]
\caption{LtM prompt, for both the baseline and unit cases.}
\begin{lstlisting}
Your task is to solve the problem below. You can reason about the problem before stating your answer. The answer MUST be between the following tags: <ANSWER>...</ANSWER>. An example is provided to showcase how to use the tags; you must only solve the last problem given.

## Examples

{examples}

## Problem

Problem: {problem}
Answer:
\end{lstlisting}
\end{listing}

\begin{listing}[H]
\caption{RDD prompt for the decomposition step with independent sub-problems.}
\begin{lstlisting}
You manage {width} workers. Your task is to decompose the problem below in order to delegate sub-problems to your workers. The decomposition must be complete: combining the solutions to the sub-problems must be enough to solve the original problem. You must be brief and clear. You must consider that all sub-problems must be solved independently and that merging their solutions should produce the solution to the original problem. Do not attempt to solve the sub-problems.

If the problem is simple enough to be solved by a single worker, you must only output "This is a unit problem". Otherwise, you must propose sub-problems in a bullet list. In each bullet point, provide all necessary information for a worker to solve the sub-problem. The workers will not be provided with the original problem description nor the other sub-problems. Therefore, you must include all necessary data and instructions in the description of each sub-problem. You must only use from one up to {width} of the workers, never more than {width} workers. The sub-problems you generate can be still complex; they will be decomposed again by your workers if necessary.

You can decompose the task via either the "data decomposition strategy" or the "task decomposition strategy":

- The data decomposition strategy produces sub-problems describing exactly the same data transformation given in the original problem, applied to partitions of the input data. The partitions of the input data must be of approximately equal size. The sub-problem descriptions must be exactly the same as the description of the original problem.
- The task decomposition strategy produces sub-problems describing different data transformations, applied to exactly the same input data given in the original problem. For example, the sub-problem transformations may describe sub-steps required to solve the original problem.

Examples are provided below to illustrate some decompositions; you must only provide a decomposition for the last problem.

## Examples

{examples}

## Problem

Problem: {problem}
Answer:
\end{lstlisting}
\end{listing}

\begin{listing}[H]
\caption{RDD prompt for the merging step.}
\begin{lstlisting}
The problem below was decomposed into sub-problems. The sub-problems and their sub-solutions are provided in bullet points below the problem. Your task is to solve the problem with the help of the sub-solutions. Often, obtaining the final solution to the problem only requires you to apply a transformation to the sub-solutions. If you find any mistakes in the sub-solutions, you can fix the mistakes while you merge the sub-solutions.

You must reason about how to merge the sub-solutions and solve the problem before stating your final answer. The final answer MUST be between the following tags: <ANSWER>...</ANSWER>. Some examples are provided to showcase how to use the tags and to illustrate some merging strategies; you must only solve the last problem.

## Examples

{examples}

## Problem

Problem: {problem}
{subsolutions}
Answer: 
\end{lstlisting}
\end{listing}

\begin{listing}[H]
\caption{RDD prompt for the decomposition step with possibly dependent sub-problems.}
\begin{lstlisting}
You manage {width} workers. Your task is to decompose the problem below in order to delegate sub-problems to your workers. You must only use from one up to {width} of the workers, never more than {width} workers. The decomposition must be complete: combining the solutions to the sub-problems must be enough to solve the original problem. You must be brief and clear. Do not attempt to solve the sub-problems.

If the problem is simple enough to be solved by a single worker, you must only output "This is a unit problem". Otherwise, you must propose sub-problems in a bullet list. The workers will not be provided with the original problem description nor the other sub-problem descriptions. Therefore, you must include all necessary data and instructions in the description of each sub-problem. You must never reference the original problem and you must not assume the workers can access its description and input data; instead, you must copy all relevant instructions and input data to the descriptions of sub-problems when necessary. The sub-problems you generate can be still complex; they will be decomposed again by your workers if necessary.

You can decompose the task via either the "data decomposition strategy" or the "task decomposition strategy":

- The data decomposition strategy produces sub-problems describing exactly the same data transformation given in the original problem, applied to partitions of the input data. The partitions of the input data must be of approximately equal size. The sub-problem descriptions must be exactly the same as the description of the original problem.
- The task decomposition strategy produces sub-problems describing different data transformations, applied to exactly the same input data given in the original problem. For example, the sub-problem transformations may describe sub-steps required to solve the original problem.

Each sub-problem must have a unique identifier given between square brackets before the sub-problem description. If you need to, you can also specify dependencies: within each sub-problem's description, you can refer to the solutions to other sub-problems using their identifiers between curly braces. Sub-problems cannot have the original problem as a dependency. The scheduler will substitute the identifiers of the dependencies with their solutions before sending the sub-problems to the workers. All dependencies stated between curly braces must also be sub-problems present in your bullet list.

The examples below illustrate some decompositions. You must only provide a decomposition for the last problem, do not attempt to decompose the examples.

[Continues as the prompt given in A.3.1.]
\end{lstlisting}
\end{listing}

\clearpage
\section{In-context examples}
\label{sec:appendix:examples}
\begin{listing}[H]
\caption{CoT examples for the letter concatenation task.}
\begin{lstlisting}
<INPUT>Concatenate using a space the characters at index 1 of each word in the list [Gladys, Rathav, Miya]; indices start at zero.</INPUT>
<TARGET>Let's think step by step. The characters at index 1 in the input words are "l", "a" and "i". If we concatenate these, we get the answer <ANSWER>"l a i"</ANSWER>.</TARGET>

<INPUT>Concatenate using a space the characters at index 3 of each word in the list [Gloria, Ricardo, Kanwar, Chon, Manoj, Enrique, Xiong, Shaw]; indices start at zero.</INPUT>
<TARGET>Let's think step by step. The characters at index 3 in the input words are "r", "a", "w", "n", "o", "i", "n" and "w". If we concatenate these, we get the answer <ANSWER>"r a w n o i n w"</ANSWER>.</TARGET>

<INPUT>Concatenate using a space the characters at index 0 of each word in the list [Olga, Cynthia, Gladys, Cynthia, Aliyu]; indices start at zero.</INPUT>
<TARGET>Let's think step by step. The characters at index 0 in the input words are "O", "C", "G", "C" and "A". If we concatenate these, we get the answer <ANSWER>"O C G C A"</ANSWER>.</TARGET>

<INPUT>Concatenate using a space the characters at index 3 of each word in the list [Wilson]; indices start at zero.</INPUT>
<TARGET>Let's think step by step. The characters at index 3 in the input words are "s". If we concatenate these, we get the answer <ANSWER>"s"</ANSWER>.</TARGET>

<INPUT>Concatenate using a space the characters at index 2 of each word in the list [Ilya, Jacques, Francesco, Samuel, Jadhav, Rivera, Irma, Jianping, Samuel, Christian]; indices start at zero.</INPUT>
<TARGET>Let's think step by step. The characters at index 2 in the input words are "y", "c", "a", "m", "d", "v", "m", "a", "m" and "r". If we concatenate these, we get the answer <ANSWER>"y c a m d v m a m r"</ANSWER>.</TARGET>
\end{lstlisting}
\end{listing}

\begin{listing}[H]
\caption{Generic CoT examples.}
\begin{lstlisting}
<INPUT>Who is younger: Michael Jordan, Cristiano Ronaldo or Usain Bolt?
- Sub-problem 1: How old is Cristiano Ronaldo? Sub-solution 1: 39 years old.
- Sub-problem 2: How old is Michael Jordan? Sub-solution 2: 61 years old.
- Sub-problem 3: How old is Usain Bolt? Sub-solution 3: 37 years old.</INPUT>
<TARGET>Let's think step by step. We must compare the ages of each person: (Michael Jordan, 61) > (Cristiano Ronaldo, 39) > (Usain Bolt, 37). The answer must be the person with the lowest age. Thus, the solution is <ANSWER>Usain Bolt</ANSWER></TARGET>

<INPUT>Peter had 3 apples, 7 oranges and 12 pears. He gave 1 apple to John, 4 oranges to Maria and 3 pears to Ana. How many pieces of fruit does Peter have left?</INPUT>
<TARGET>Let's think step by step. If Peter has 3 apples and gives 1 to John, he will lose 1 apple. If Peter has 7 oranges and gives 4 to Maria, he will lose 4 oranges. If Peter has 12 pears and gives 3 to Ana, he will lose 3 pears. Thus, the solution is 3 - 1 + 7 - 4 + 12 - 3 = <ANSWER>14</ANSWER></TARGET>

<INPUT>What is ((((5 + 4) * 100) + 267) / (3 * 10))?</INPUT>
<TARGET>Let's think step by step. 5 + 4 = 9. 9 * 100 = 900. 900 + 267 = 1167. 3 * 10 = 30. Thus, the answer is 1167 / 30 = <ANSWER>38.9</ANSWER></TARGET>

<INPUT>Which word in the list [hush, oceanic, surge, present, lie, wry, giraffe, dine, guide, urge, complete, tasteless, glorious, bird, raspy, murky, zoom, juice, select, liquid, hope, install, complete, aromatic, oceanic, fish, excited, fabricator, internal, dinosaurs, noiseless, partner] is longer?</INPUT>
<TARGET>Let's think step by step. The lengths of each word are (hush, 4), (oceanic, 7), (surge, 5), (present, 7), (lie, 3), (wry, 3), (giraffe, 7), (dine, 4), (guide, 5), (urge, 4), (complete, 8), (tasteless, 9), (glorious, 8), (bird, 4), (raspy, 5), (murky, 5), (zoom, 4), (juice, 5), (select, 6), (liquid, 6), (hope, 4), (install, 7), (complete, 8), (aromatic, 8), (oceanic, 7), (fish, 4), (excited, 7), (fabricator, 10), (internal, 8), (dinosaurs, 9), (noiseless, 9) and (partner, 7). Thus, the solution is <ANSWER>fabricator</ANSWER></TARGET>

<INPUT>Is the following sports-related sentence plausible? "Joao Moutinho caught the screen pass in the NFC championship."</INPUT>
<TARGET>Joao Moutinho is a soccer player. The NFC championship is part of American football, not soccer. Thus, the answer is <ANSWER>no</ANSWER></TARGET>
\end{lstlisting}
\end{listing}

\begin{listing}[H]
\caption{LtM examples for the letter concatenation task.}
\begin{lstlisting}
<INPUT>Concatenate using a space the characters at index 1 of each word in the list [Gladys, Rathav, Miya]; indices start at zero.</INPUT>
<TARGET>The letters at index 1 of "Gladys" and "Rathav" are "l" and "a". Concatenating "l" and "a" leads to "l a". The letter at index 1 of "Miya" is "i". Concatenating "l a" and "i" leads to <ANSWER>"l a i"</ANSWER>.</TARGET>

<INPUT>Concatenate using a space the characters at index 3 of each word in the list [Gloria, Ricardo, Kanwar, Chon, Manoj, Enrique, Xiong, Shaw]; indices start at zero.</INPUT>
<TARGET>The letters at index 3 of "Gloria" and "Ricardo" are "r" and "a". Concatenating "r" and "a" leads to "r a". The letter at index 3 of "Kanwar" is "w". Concatenating "r a" and "w" leads to "r a w". The letter at index 3 of "Chon" is "n". Concatenating "r a w" and "n" leads to "r a w n". The letter at index 3 of "Manoj" is "o". Concatenating "r a w n" and "o" leads to "r a w n o". The letter at index 3 of "Enrique" is "i". Concatenating "r a w n o" and "i" leads to "r a w n o i". The letter at index 3 of "Xiong" is "n". Concatenating "r a w n o i" and "n" leads to "r a w n o i n". The letter at index 3 of "Shaw" is "w". Concatenating "r a w n o i n" and "w" leads to <ANSWER>"r a w n o i n w"</ANSWER>.</TARGET>

<INPUT>Concatenate using a space the characters at index 0 of each word in the list [Olga, Cynthia, Gladys, Cynthia, Aliyu]; indices start at zero.</INPUT>
<TARGET>The letters at index 0 of "Olga" and "Cynthia" are "O" and "C". Concatenating "O" and "C" leads to "O C". The letter at index 0 of "Gladys" is "G". Concatenating "O C" and "G" leads to "O C G". The letter at index 0 of "Cynthia" is "C". Concatenating "O C G" and "C" leads to "O C G C". The letter at index 0 of "Aliyu" is "A". Concatenating "O C G C" and "A" leads to <ANSWER>"O C G C A"</ANSWER>.</TARGET>

<INPUT>Concatenate using a space the characters at index 3 of each word in the list [Wilson]; indices start at zero.</INPUT>
<TARGET>The letter at index 3 of "Wilson" is <ANSWER>"s"</ANSWER>.</TARGET>

<INPUT>Concatenate using a space the characters at index 2 of each word in the list [Ilya, Jacques, Francesco, Samuel, Jadhav, Rivera, Irma, Jianping, Samuel, Christian]; indices start at zero.</INPUT>
<TARGET>The letters at index 2 of "Ilya" and "Jacques" are "y" and "c". Concatenating "y" and "c" leads to "y c". The letter at index 2 of "Francesco" is "a". Concatenating "y c" and "a" leads to "y c a". The letter at index 2 of "Samuel" is "m". Concatenating "y c a" and "m" leads to "y c a m". The letter at index 2 of "Jadhav" is "d". Concatenating "y c a m" and "d" leads to "y c a m d". The letter at index 2 of "Rivera" is "v". Concatenating "y c a m d" and "v" leads to "y c a m d v". The letter at index 2 of "Irma" is "m". Concatenating "y c a m d v" and "m" leads to "y c a m d v m". The letter at index 2 of "Jianping" is "a". Concatenating "y c a m d v m" and "a" leads to "y c a m d v m a". The letter at index 2 of "Samuel" is "m". Concatenating "y c a m d v m a" and "m" leads to "y c a m d v m a m". The letter at index 2 of "Christian" is "r". Concatenating "y c a m d v m a m" and "r" leads to <ANSWER>"y c a m d v m a m r"</ANSWER>.</TARGET>
\end{lstlisting}
\end{listing}

\begin{listing}[H]
\caption{RDD examples for the decomposition step and the letter concatenation task.}
\begin{lstlisting}
<INPUT>Concatenate using a space the characters at index 1 of each word in the list [Dong]; indices start at zero.</INPUT>
<TARGET>This is a unit problem.</TARGET>

<INPUT>Concatenate using a space the characters at index 2 of each word in the list [Shimizu, Hoang, Muhammad, Mejia, Fernandes, Punam, Cesar]; indices start at zero.</INPUT>
<TARGET>- Concatenate using a space the characters at index 2 of each word in the list [Shimizu, Hoang, Muhammad, Mejia]; indices start at zero.
- Concatenate using a space the characters at index 2 of each word in the list [Fernandes, Punam, Cesar]; indices start at zero.</TARGET>

<INPUT>Concatenate using a space the characters at index 2 of each word in the list [Lawal, Jadhav, Sekha, Jadhav, Abraham, Sushila, Hoang, Gerhard, Heinz]; indices start at zero.</INPUT>
<TARGET>- Concatenate using a space the characters at index 2 of each word in the list [Lawal, Jadhav, Sekha, Jadhav]; indices start at zero.
- Concatenate using a space the characters at index 2 of each word in the list [Abraham, Sushila, Hoang, Gerhard]; indices start at zero.
- Concatenate using a space the characters at index 2 of each word in the list [Heinz]; indices start at zero.</TARGET>

<INPUT>Concatenate using a space the characters at index 2 of each word in the list [Kailash, Ouattara, Kasongo, Perez, Jyoti]; indices start at zero.</INPUT>
<TARGET>This is a unit problem.</TARGET>

<INPUT>Concatenate using a space the characters at index 0 of each word in the list [Guan, Madina, Mejia, Herrera, Christopher, Sergey, Karina, Lucy, Ortega, Vera, Mallik, Weimin, Kwon, Zhan, Shaw, Tahir, Chang, Halyna, Weidong, Ochoa, Dung, George, Nayak, Jianming, Paola, Awad, Nabil, Garba, Amal, Sergey, Mustapha, Garcia, Bello, Sergey, Otieno, Rojas, Andrew, Mustafa, Haji, Philip, Leticia, Syed, Blanca, Mahendra, Salim, Ghulam, Quan, Yanhua, Artyom, Muhammad]; indices start at zero.</INPUT>
<TARGET>- Concatenate using a space the characters at index 0 of each word in the list [Guan, Madina, Mejia, Herrera, Christopher, Sergey, Karina, Lucy, Ortega, Vera, Mallik, Weimin]; indices start at zero.
- Concatenate using a space the characters at index 0 of each word in the list [Kwon, Zhan, Shaw, Tahir, Chang, Halyna, Weidong, Ochoa, Dung, George, Nayak, Jianming]; indices start at zero.
- Concatenate using a space the characters at index 0 of each word in the list [Paola, Awad, Nabil, Garba, Amal, Sergey, Mustapha, Garcia, Bello, Sergey, Otieno, Rojas]; indices start at zero.
- Concatenate using a space the characters at index 0 of each word in the list [Andrew, Mustafa, Haji, Philip, Leticia, Syed, Blanca, Mahendra, Salim, Ghulam, Quan, Yanhua, Artyom, Muhammad]; indices start at zero.</TARGET>
\end{lstlisting}
\end{listing}

\begin{listing}[H]
\caption{Generic RDD examples for the decomposition step.}
\begin{lstlisting}
<INPUT>If Peter has 3 apples and gives 1 to John, how many apples does Peter have left?</INPUT>
<TARGET>This problem is simple enough to be solved directly by a single mathematical operation. <ANSWER>This is a unit problem.</ANSWER></TARGET>

<INPUT>What is ((((5 + 4) * 100) + 267) / (3 * 10))?</INPUT>
<TARGET>We can use the data decomposition strategy here by splitting the input formula into sub-formulas. We can use two workers. The merged solution will be $sub_solution_1 / sub_solution_2$.
<ANSWER>- What is ((5 + 4) * 100) + 267?
- What is 3 * 10?</ANSWER></TARGET>

<INPUT>What is the result of log_2(16)?</INPUT>
<TARGET>This problem is simple enough to be solved directly by a single mathematical operation. <ANSWER>This is a unit problem.</ANSWER></TARGET>

<INPUT>Write the blueprint for a webpage view using the Vue3 framework about a study on salaries based on profession and age. The view must contain an initial text description of the study, a table with headers "Name", "Age", "Profession" and "Salary", as well as a picture slider. The data for the table will be available from a local JSON file, and the pictures for the slider will also be available locally.</INPUT>
<TARGET>We can use the task decomposition strategy here by splitting the task into smaller independent tasks, consisting on creating Vue3 components for each element of the view. We can use two workers. The merged solution will be the code for the components generated when solving the sub-problems, as well as code for the view using such components.
<ANSWER>- Write code using the Vue3 framework for a component representing a table with headers "Name", "Age", "Profession" and "Salary". The data for the table will be available from a local JSON file.
- Write code using the Vue3 framework for a component representing a picture slider. The pictures for the slider will be available locally.</ANSWER></TARGET>

<INPUT>Write a Python function that takes the base and height of a triangle (two floating point numbers) and returns its area (also a floating point number).</INPUT>
<TARGET>This problem is simple enough to be solved directly by writing a short Python function. <ANSWER>This is a unit problem.</ANSWER></TARGET>

<INPUT>Which word in the list [hush, oceanic, surge, present, lie, wry, giraffe, dine, guide, urge, complete, tasteless, glorious, bird, raspy, murky, zoom, juice, select, liquid, hope, install, complete, aromatic, oceanic, fish, excited, tail, internal, dinosaurs, noiseless, partner] is longer? If there is more than one word with the same length, any of them is a valid answer.</INPUT>
<TARGET>We can use the data decomposition strategy here by splitting the input list of words into smaller lists. We can use three workers. Each of the list partitions will be approximately the same size. The merged solution will be the longest word out of all the sub-solutions.
\end{lstlisting}
\end{listing}

\begin{listing}[H]
\caption{Generic RDD examples for the decomposition step (continued).}
\begin{lstlisting}
<ANSWER>- Which word in the list [hush, oceanic, surge, present, lie, wry, giraffe, dine, guide, urge, complete] is longer? If there is more than one word with the same length, any of them is a valid answer.
- Which word in the list [tasteless, glorious, bird, raspy, murky, zoom, juice, select, liquid, hope, install] is longer? If there is more than one word with the same length, any of them is a valid answer.
- Which word in the list [complete, aromatic, oceanic, fish, excited, tail, internal, dinosaurs, noiseless, partner] is longer? If there is more than one word with the same length, any of them is a valid answer.</ANSWER></TARGET>

<INPUT>Which word in the list [cow, banana, ensemble, castle, wise] is longer?</INPUT>
<TARGET>This problem is simple enough to be solved directly by performing a length comparison of only two words. <ANSWER>This is a unit problem.</ANSWER></TARGET>

<INPUT>Is the following sports-related sentence plausible? "Joao Moutinho caught the screen pass in the NFC championship."</INPUT>
<TARGET>We can use the task decomposition strategy here by proposing questions to gather the information required to solve the original problem. We can use two workers. The merged solution will be "yes" if the resulting information of both Joao Moutinho and the NFC championship match, and "no" otherwise.
<ANSWER>- Which sport does Joao Moutinho play?
- To which sport does the NFC championship belong to?</ANSWER></TARGET>
\end{lstlisting}
\end{listing}

\begin{listing}[H]
\caption{RDD examples for the merging step and the letter concatenation task.}
\begin{lstlisting}
<INPUT>Concatenate using a space the characters at index 1 of each word in the list [Orlando, Arif, Keith, Lyudmyla, Amin, Theresa, Stefan, Gilberto, Samina, Yoko, Katarzyna, Haiying, Saraswati, Theresa, Bernadette, Maung, Lopez, Pereira, Shaikh, Brown, Ortiz]; indices start at zero.
- Sub-problem 1: Concatenate using a space the characters at index 1 of each word in the list [Orlando, Arif, Keith, Lyudmyla, Amin]; indices start at zero. Sub-solution 1: "r r e y m".
- Sub-problem 2: Concatenate using a space the characters at index 1 of each word in the list [Theresa, Stefan, Gilberto, Samina, Yoko]; indices start at zero. Sub-solution 2: "h t i a o".
- Sub-problem 3: Concatenate using a space the characters at index 1 of each word in the list [Katarzyna, Haiying, Saraswati, Theresa, Bernadette]; indices start at zero. Sub-solution 3: "a a a h e".
- Sub-problem 4: Concatenate using a space the characters at index 1 of each word in the list [Maung, Lopez, Pereira, Shaikh, Brown, Ortiz]; indices start at zero. Sub-solution 4: "a o e h r r".</INPUT>
<TARGET>"r r e y m h t i a o a a a h e a o e h r r"</TARGET>

<INPUT>Concatenate using a space the characters at index 2 of each word in the list [Lawal, Jadhav, Sekha, Jadhav, Abraham, Sushila, Hoang, Gerhard, Heinz]; indices start at zero.
- Sub-problem 1: Concatenate using a space the characters at index 2 of each word in the list [Lawal, Jadhav, Sekha, Jadhav]; indices start at zero. Sub-solution 1: "w d k d".
- Sub-problem 2: Concatenate using a space the characters at index 2 of each word in the list [Abraham, Sushila, Hoang, Gerhard]; indices start at zero. Sub-solution 2: "r s a r".
- Sub-problem 3: Concatenate using a space the characters at index 2 of each word in the list [Heinz]; indices start at zero. Sub-solution 3: "i".</INPUT>
<TARGET>"w d k d r s a r i"</TARGET>

<INPUT>Concatenate using a space the characters at index 1 of each word in the list [Prem, Wilson, Ashraf, Gilberto, Shobha]; indices start at zero.
- Sub-problem 1: Concatenate using a space the characters at index 1 of each word in the list [Prem, Wilson, Ashraf]; indices start at zero. Sub-solution 1: "r i s".
- Sub-problem 2: Concatenate using a space the characters at index 1 of each word in the list [Gilberto, Shobha]; indices start at zero. Sub-solution 2: "i h".</INPUT>
<TARGET>"r i s i h"</TARGET>
\end{lstlisting}
\end{listing}

\begin{listing}[H]
\caption{RDD examples for the merging step and the letter concatenation task (continued).}
\begin{lstlisting}
<INPUT>Concatenate using a space the characters at index 1 of each word in the list [Robin, Mostafa, Hadi, Gutierrez, Farooq, Nicolas, Alicia, Sandra, Xiaolin, Valerie]; indices start at zero.
- Sub-problem 1: Concatenate using a space the characters at index 1 of each word in the list [Robin, Mostafa, Hadi]; indices start at zero. Sub-solution 1: "o o a".
- Sub-problem 2: Concatenate using a space the characters at index 1 of each word in the list [Gutierrez, Farooq, Nicolas]; indices start at zero. Sub-solution 2: "u a i".
- Sub-problem 3: Concatenate using a space the characters at index 1 of each word in the list [Alicia, Sandra, Xiaolin, Valerie]; indices start at zero. Sub-solution 3: "l a i a".</INPUT>
<TARGET>"o o a u a i l a i a"</TARGET>

<INPUT>Concatenate using a space the characters at index 1 of each word in the list [Cheng, Jianwei, Magdalena, Raimundo, Rosario, Raju, Orlando]; indices start at zero.
- Sub-problem 1: Concatenate using a space the characters at index 1 of each word in the list [Cheng, Jianwei, Magdalena]; indices start at zero. Sub-solution 1: "h i a".
- Sub-problem 2: Concatenate using a space the characters at index 1 of each word in the list [Raimundo, Rosario, Raju]; indices start at zero. Sub-solution 2: "a o a".
- Sub-problem 3: Concatenate using a space the characters at index 1 of each word in the list [Orlando]; indices start at zero. Sub-solution 3: "r".</INPUT>
<TARGET>"h i a a o a r"</TARGET>
\end{lstlisting}
\end{listing}

\begin{listing}[H]
\caption{Generic RDD examples for the merging step.}
\begin{lstlisting}
<INPUT>Who is younger: Michael Jordan, Cristiano Ronaldo or Usain Bolt?
- Sub-problem 1: How old is Cristiano Ronaldo? Sub-solution 1: 39 years old.
- Sub-problem 2: How old is Michael Jordan? Sub-solution 2: 61 years old.
- Sub-problem 3: How old is Usain Bolt? Sub-solution 3: 37 years old.</INPUT>
<TARGET>We can obtain the solution to the original problem by comparing the ages given in the sub-solutions. Thus, the solution is <ANSWER>Usain Bolt</ANSWER></TARGET>

<INPUT>Peter had 3 apples, 7 oranges and 12 pears. He gave 1 apple to John, 4 oranges to Maria and 3 pears to Ana. How many pieces of fruit does Peter have left?
- Sub-problem 1: Peter had 3 apples and gave 1 to John. How many apples does Peter have left? Sub-solution 1: 2.
- Sub-problem 2: Peter had 7 oranges and gave 4 to Maria. How many oranges does Peter have left? Sub-solution 2: 3.
- Sub-problem 3: Peter had 12 pears and gave 3 to Ana. How many pears does Peter have left? Sub-solution 3: 9.</INPUT>
<TARGET>We can obtain the solution to the original problem by adding up the pieces Peter has left for each type of fruit. These pieces are given by each sub-solution. Thus, the solution is 2 + 3 + 9 = <ANSWER>14</ANSWER></TARGET>

<INPUT>What is ((((5 + 4) * 100) + 267) / (3 * 10))?
- Sub-problem 1: What is ((5 + 4) * 100) + 267? Sub-solution 1: 1167.
- Sub-problem 2: What is 3 * 10? Sub-solution 2: 30.</INPUT>
<TARGET>We can obtain the solution to the original problem by performing the operation $sub_solution_1 / sub_solution_2$. Thus, the solution is 1167 / 30 = <ANSWER>38.9</ANSWER></TARGET>

<INPUT>Which word in the list [hush, oceanic, surge, present, lie, wry, giraffe, dine, guide, urge, complete, tasteless, glorious, bird, raspy, murky, zoom, juice, select, liquid, hope, install, complete, aromatic, oceanic, fish, excited, fabricator, internal, dinosaurs, noiseless, partner] is longer?
- Sub-problem 1: Which word in the list [hush, oceanic, surge, present, lie, wry, giraffe, dine, guide, urge, complete] is longer? Sub-solution 1: complete.
- Sub-problem 2: Which word in the list [tasteless, glorious, bird, raspy, murky, zoom, juice, select, liquid, hope, install] is longer? Sub-solution 2: tasteless.
- Sub-problem 3: Which word in the list [complete, aromatic, oceanic, fish, excited, fabricator, internal, dinosaurs, noiseless, partner] is longer? Sub-solution 3: fabricator.</INPUT>
<TARGET>We can obtain the solution to the original problem by taking the longest word out of the three sub-solutions. "complete" has 8 letters, "tasteless" has 9 letters and "fabricator" has 10 letters. Thus, the solution is <ANSWER>fabricator</ANSWER></TARGET>

<INPUT>Is the following sports-related sentence plausible? "Joao Moutinho caught the screen pass in the NFC championship."
- Sub-problem 1: Which sport does Joao Moutinho play? Sub-solution 1: Soccer.
- Sub-problem 2: To which sport does the NFC championship belong to? Sub-solution 2: American football.</INPUT>
<TARGET>Joao Moutinho does not play the same sport that the NFC championship belongs to. Thus, the answer is <ANSWER>no</ANSWER></TARGET>
\end{lstlisting}
\end{listing}

\begin{listing}[H]
\caption{Generic RDD examples for the decomposition step and the length reversal task.}
\begin{lstlisting}
<INPUT>If Peter has 3 apples and gives 1 to John, how many apples does Peter have left?</INPUT>
<TARGET>This problem is simple enough to be solved directly by a single mathematical operation. <ANSWER>This is a unit problem.</ANSWER></TARGET>

<INPUT>Who is the brother of the Sultan of Brunei married to?</INPUT>
<TARGET>We can use the task decomposition strategy here by splitting the task into smaller tasks, in order to find out the necessary information to answer the main question. We can use two workers. The merged solution will be attained by using the intermediate information to solve the original question.
<ANSWER>- [P-1] Who is the Sultan of Brunei?
- [P-2] Who is the brother of {P-1}?
- [P-3] Who is married to {P-2}?</ANSWER></TARGET>

<INPUT>Solve for $y$: $\frac{4}{\log_2 8}x = \log_2 16 + 7$, $y = 3x$.</INPUT>
<TARGET>We can use the task decomposition strategy here by splitting the task into simpler mathematical operations. We can use three workers. The merged solution will be attained by using the intermediate results to obtain a value for $y$.
<ANSWER>- [P-1] What is the result of $\frac{4}{\log_2 8}$?
- [P-2] What is the result of $\log_2 16 + 7$?
- [P-3] What is result of $\frac{{P-2}}{{P-1}}$?</ANSWER></TARGET>

<INPUT>Write the blueprint for a webpage view using the Vue3 framework about a study on salaries based on profession and age. The view must contain an initial text description of the study, a table with headers "Name", "Age", "Profession" and "Salary", as well as a picture slider. The data for the table will be available from a local JSON file, and the pictures for the slider will also be available locally.</INPUT>
<TARGET>We can use the task decomposition strategy here by splitting the task into smaller tasks, consisting on creating Vue3 components for each element of the view. We can use two workers. The merged solution will be the code for the components generated when solving the sub-problems, as well as code for the view using such components.
<ANSWER>- [P-1] Write code using the Vue3 framework for a component representing a table with headers "Name", "Age", "Profession" and "Salary". The data for the table will be available from a local JSON file.
- [P-2] Write code using the Vue3 framework for a component representing a picture slider. The pictures for the slider will be available locally.</ANSWER></TARGET>

<INPUT>Write a Python function that takes the base and height of a triangle (two floating point numbers) and returns its area (also a floating point number).</INPUT>
<TARGET>This problem is simple enough to be solved directly by writing a short Python function. <ANSWER>This is a unit problem.</ANSWER></TARGET>
\end{lstlisting}
\end{listing}

\begin{listing}[H]
\caption{Generic RDD examples for the decomposition step and the length reversal task (continued).}
\begin{lstlisting}
<INPUT>Which is the oldest country out of Germany, Japan, Switzerland, Spain, Bolivia, Angola, Laos, Belgium, Canada, Mexico, Costa Rica, Indonesia, Pakistan and Rwanda?</INPUT>
<TARGET>We can use the data decomposition strategy here by splitting the input data into evenly sized partitions and solving the same problem for each partition. We can use two workers. The merged solution will be oldest country out of all the sub-solutions.
<ANSWER>- [P-1] Which is the oldest country out of Germany, Japan, Switzerland, Spain, Bolivia, Angola and Laos?
- [P-2] Which is the oldest country out of Belgium, Canada, Mexico, Costa Rica, Indonesia, Pakistan and Rwanda?</ANSWER></TARGET>

<INPUT>Which is the oldest country out of Germany, Japan, Switzerland, Spain, Bolivia, Angola and Laos?</INPUT>
<TARGET>We can use the task decomposition strategy here by performing different steps to obtain all required information to answer the question. We can use two workers. The merged solution will be the longest word out of all the sub-solutions.
<ANSWER>- [P-1] Create a list of country-age pairs for each country and their respective ages out of Germany, Japan, Switzerland, Spain, Bolivia, Angola and Laos.
- [P-2] Which is the country with the largest age, given the following list of country-age pairs: {P-1}?</ANSWER></TARGET>

<INPUT>Which word in the list [cow, banana, ensemble, castle, wise] is longer?</INPUT>
<TARGET>This problem is simple enough to be solved directly by performing a length comparison of only five words. <ANSWER>This is a unit problem.</ANSWER></TARGET>
\end{lstlisting}
\end{listing}

\clearpage
\section{Resource usage statistics}
\label{sec:appendix:resource-usage}
We attempted to match the estimated resource usage of the baselines and our method by the amount of Self-Consistency (SC) \citep{wangSelfConsistencyImprovesChain2022} samples. We used the following formula for resource matching:  $\text{n\_context\_tokens} + 3 \cdot \text{n\_output\_tokens}$.

\begin{table}[h]
    \centering
    \begin{tabular}{p{0.75cm}p{2cm}p{1.25cm}p{1.25cm}p{2.7cm}p{2.7cm}}
        \toprule
        \textbf{$n_0$}&\textbf{Method}&\textbf{Time}&\textbf{Calls}&\textbf{Context tokens}&\textbf{Output tokens}\\
        \midrule
        5 &CoT+SC&2.75h &2,500 &1,249,529 &84,656\\
          &LtM+SC&3.78h &1,100 &1,141,800 &122,996\\
          &RDD+LtM&0.53h &506   &466,812   &14,572\\
        \midrule
        10&CoT+SC&3.80h &2,500 &1,332,860 &120,837\\
          &LtM+SC&8.87h &1,100 &1,373,365 &295,972\\
          &RDD+LtM&1.20h &880   &817,621   &35,440\\
        \midrule
        20&CoT+SC&5.82h &2,500 &1,497,807 &188,666\\
          &LtM+SC&11.15h&500   &889,367   &378,866\\
          &RDD+LtM&2.50h &1,541 &1,416,940 &75,390\\
        \midrule
        50&CoT+SC&13.18h&2,700 &2,166,008 &437,959\\
          &LtM+SC&17.78h* &700 &1,893,349 &899,997\\
          &RDD+LtM&4.85h &2,022 &2,712,931 &250,710\\
        \midrule
        70&CoT+SC&12,18h* &2,700 &2,527,744 &653,758 \\
          &LtM+SC&66.94h* &509 &754,192 &3,098,360\\
          &RDD+LtM&7.10h &924   &1,289,806 &389,013\\
        \midrule
        90&CoT+SC&25.57h* &2,700 &3,536,975 &1,372,692 \\
          &LtM+SC&310.12h* &421 &686,331 &3,121,383\\
          &RDD+LtM&10.53h &974  &1,396,950 &570,232\\
        \bottomrule \\
    \end{tabular}
    \caption{Resource usage for the letter concatenation benchmark with task-specific examples. Experiments were run with NVIDIA A100 GPUs; those experiments marked with an asterisk were run with NVIDIA H100 GPUs instead.}
    \label{tab:resource-usage:letterconcat-specific}
\end{table}

\begin{table}[h]
    \centering
    \begin{tabular}{p{0.75cm}p{2cm}p{1.25cm}p{1.25cm}p{2.7cm}p{2.7cm}}
        \toprule
        \textbf{$n_0$}&\textbf{Method}&\textbf{Time}&\textbf{Calls}&\textbf{Context tokens}&\textbf{Output tokens}\\
        \midrule
        5 &CoT+SC&2.75h &2,500 &1,903,006 &135,905\\
          &RDD+CoT&0.57h &400   &436,868   &26,380\\
        \midrule
        10&CoT+SC&3.93h &2,500 &2,024,457 &200,147\\
          &RDD+CoT&0.63h &400   &442,718   &31,960\\
        \midrule
        20&CoT+SC&6.38h &2,500 &2,272,076 &331,999\\
          &RDD+CoT&1.83h &1,000 &1,141,523 &90,086\\
        \midrule
        50&CoT+SC&15.03h&2,700 &3,301,632 &810,976\\
          &RDD+CoT&3.48h &1,700 &1,915,870 &171,422\\
        \midrule
        70&CoT+SC&15.09h&2,200 &3,030,226 &843,774\\
          &RDD+CoT&5.32h &2,600 &2,854,574 &241,592\\
        \midrule
        90&CoT+SC&21.30h&2,200 &3,550,609 &1,128,185\\
          &RDD+CoT&5.77h &2,562 &2,883,035 &290,143\\
        \bottomrule \\
    \end{tabular}
    \caption{Resource usage for the letter concatenation benchmark with generic examples.}
    \label{tab:resource-usage:letterconcat-generic}
\end{table}

\begin{table}[h]
    \centering
    \begin{tabular}{p{0.75cm}p{2cm}p{1.25cm}p{1.25cm}p{2.7cm}p{2.7cm}}
        \toprule
        \textbf{$n_0$}&\textbf{Method}&\textbf{Time}&\textbf{Calls}&\textbf{Context tokens}&\textbf{Output tokens}\\
        \midrule
        3 &CoT+SC&1.90h &1,500 &1,238,293 &97,128 \\
          &RDD+CoT&1.28h &1,012 &1,095,336 &62,267 \\
        \midrule
        5&CoT+SC&2.49h &1,500 &1,641,229 &127,839\\
          &RDD+CoT&1.42h &900   &1,003,953 &72,215 \\
        \midrule
        7&CoT+SC&3.13h &1,500 &1,332,821 &164,456\\
          &RDD+CoT&1.70h &902   &1,015,943 &86,072 \\
        \midrule
        10&CoT+SC&3.72h &1,500 &1,742,526 &194,890\\
          &RDD+CoT&2.05h &900   &1,028,783 &103,765\\
        \midrule
        15&CoT+SC&6.02h &1,500 &1,545,019 &319,041\\
          &RDD+CoT&2.60h &900   &1,053,831 &135,297\\
        \midrule
        20&CoT+SC&6.28h &1,500 &1,945,648 &335,419\\
          &RDD+CoT&3.03h &900   &1,079,213 &159,455\\
        \bottomrule \\
    \end{tabular}
    \caption{Resource usage for the length reversal benchmark with generic examples and RDD.}
    \label{tab:resource-usage:lengthrev-deps-generic}
\end{table}

\clearpage
\section{Error analysis data}
\label{sec:appendix:error-analysis-data}
\begin{table}[ht]
    \centering
    \captionsetup[subfigure]{justification=centering}
    \begin{subfigure}[t]{0.49\linewidth}
        \centering
        \begin{tabular}{p{0.6cm}p{0.8cm}p{0.8cm}p{0.8cm}p{0.8cm}}
            \toprule
            $n_0$&$\phi_{\text{d}}$&$\phi_{\text{m}}$&$\phi_{\text{u}}$&$\phi_{\text{RDD}}$\\
            \midrule
            5  & 1.00 & 0.99 & 0.99 & 0.98 \\
            10 & 1.00 & 0.98 & 0.97 & 0.91 \\
            20 & 1.00 & 0.97 & 0.98 & 0.85 \\
            50 & 1.00 & 0.97 & 0.98 & 0.80 \\
            70 & 1.00 & 0.96 & 0.98 & 0.84 \\
            90 & 0.94 & 0.94 & 0.87 & 0.45 \\
            \bottomrule
        \end{tabular}
        \caption{Error analysis data for the task-specific in-context setting.}
        \label{tab:evaluation:errors:letterconcat-specific}
    \end{subfigure}
    \hfill
    \begin{subfigure}[t]{0.49\linewidth}
        \centering
        \begin{tabular}{p{0.6cm}p{0.8cm}p{0.8cm}p{0.8cm}p{0.8cm}}
            \toprule
            $n_0$&$\phi_{\text{d}}$&$\phi_{\text{m}}$&$\phi_{\text{u}}$&$\phi_{\text{RDD}}$\\
            \midrule
            5  & 1.00 & 0.96 & 0.97 & 0.93 \\
            10 & 1.00 & 0.99 & 0.93 & 0.85 \\
            20 & 1.00 & 0.96 & 0.92 & 0.71 \\
            50 & 1.00 & 0.93 & 0.92 & 0.42 \\
            70 & 1.00 & 0.85 & 0.93 & 0.28 \\
            90 & 1.00 & 0.81 & 0.90 & 0.11 \\
            \bottomrule
        \end{tabular}
        \caption{Error analysis data for the generic in-context setting.}
        \label{tab:evaluation:errors:letterconcat-generic}
    \end{subfigure}

    \caption{A complete data account for the analysis provided in \cref{sec:evaluation:error-analysis}. The data for the task-specific in-context experiment is given in \textbf{(a)} and the one for the generic in-context experiment in \textbf{(b)}.}
    \label{tab:evaluation:errors}
\end{table}

\clearpage
\section{Example of error propagation}
\label{sec:appendix:error-propagation}
\begin{figure}[ht]
    \centering
    \includegraphics[width=1.0\linewidth,keepaspectratio]{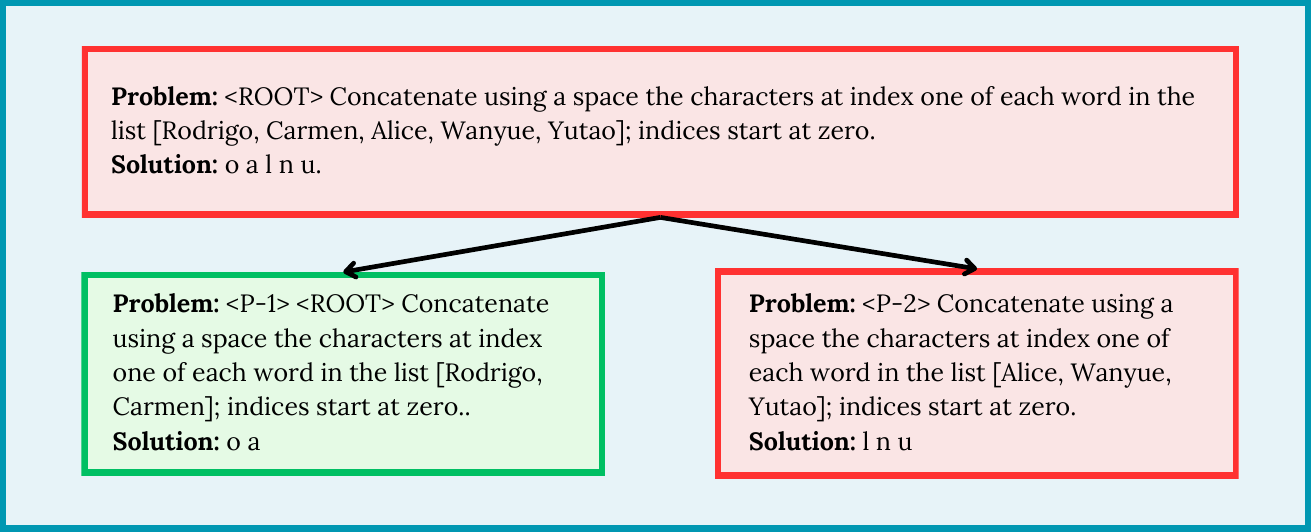} 
    \caption{Example of error propagation behavior during the execution of RDD. Green nodes correspond to correctly solved problems and red nodes to incorrectly solved problems. The method performs a mistake when unit-solving \textit{P-2}, which is carried over to the solution of the root problem.}
    \label{fig:error-propagation}
\end{figure}

\clearpage
\section{Example of error recovery}
\label{sec:appendix:error-recovery}
\begin{figure}[ht]
    \centering
    \includegraphics[width=1.0\linewidth,keepaspectratio]{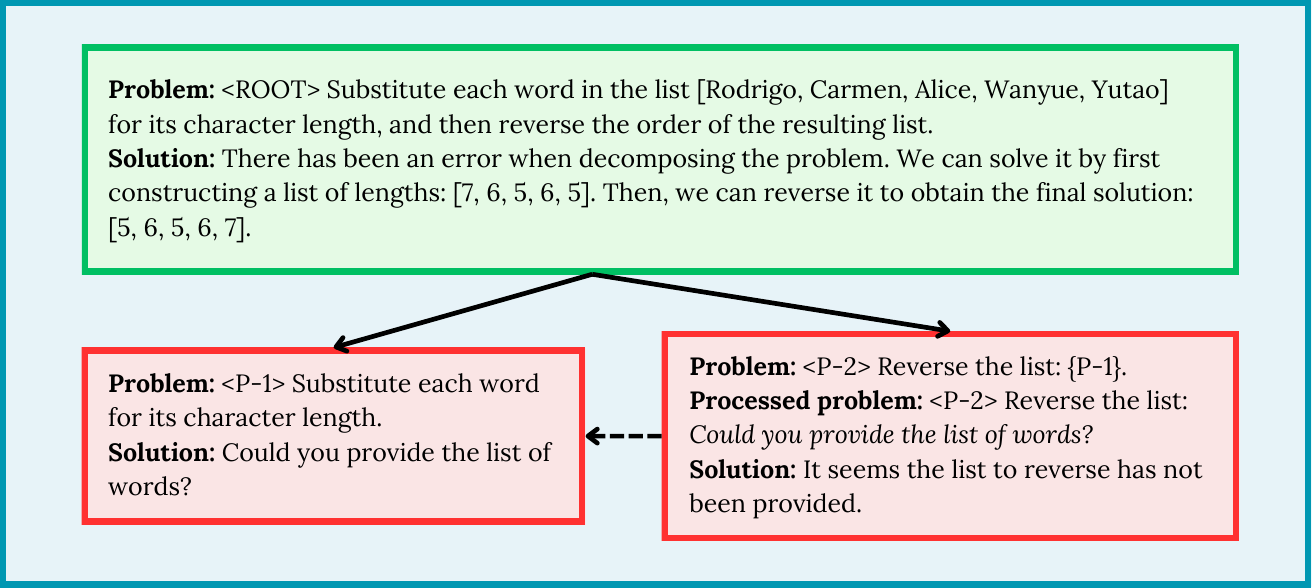}
    \caption{Example of error recovery behavior during the execution of RDD. Green nodes correspond to correctly solved problems and red nodes to incorrectly solved problems. The method does not perform the decomposition step correctly as \textit{P-1} is formulated with missing data. This issue is carried over to \textit{P-2}, but the merge step in the root problem recovers from this error.}
    \label{fig:error-recovery}
\end{figure}

\clearpage
\section{Example execution of the RDD method}
\label{sec:appendix:example-run}
An example execution of the \textsc{ScheduleBFS} procedure (\cref{alg:solve-bfs}). The execution can be followed from top to bottom. On the top-left edge of each image, we provide the type of meta-task that is performed in each step, as well as the node to which it is applied. Orange borders express a decomposition transformation of the node or an embedding of the solutions of dependencies. Green borders represent the solving process of the node, either via unit-solving or merging.

\begin{figure}[H]
    \centering
    \includegraphics[width=0.97\linewidth,keepaspectratio]{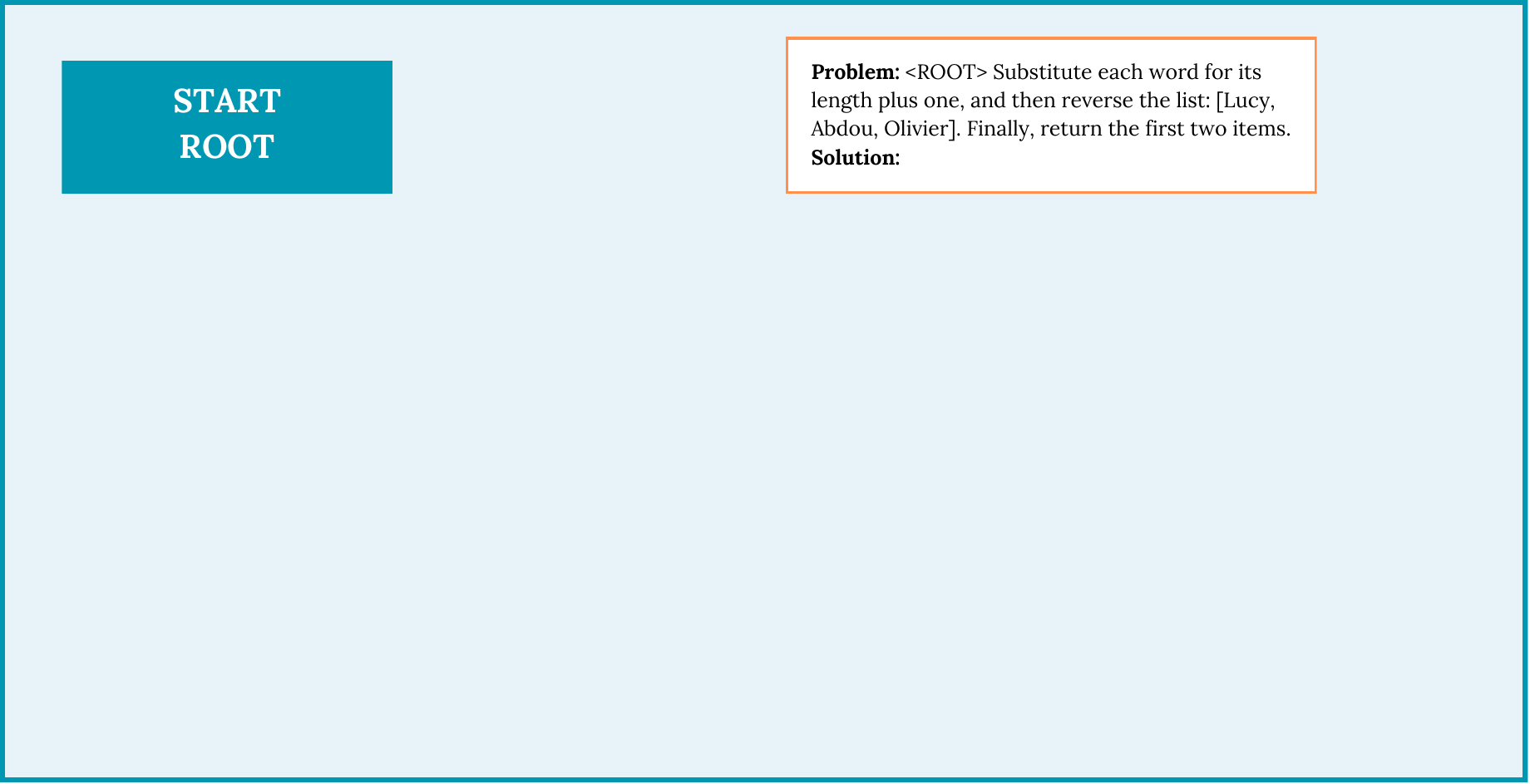}
\end{figure}

\begin{figure}[H]
    \centering
    \includegraphics[width=0.97\linewidth,keepaspectratio]{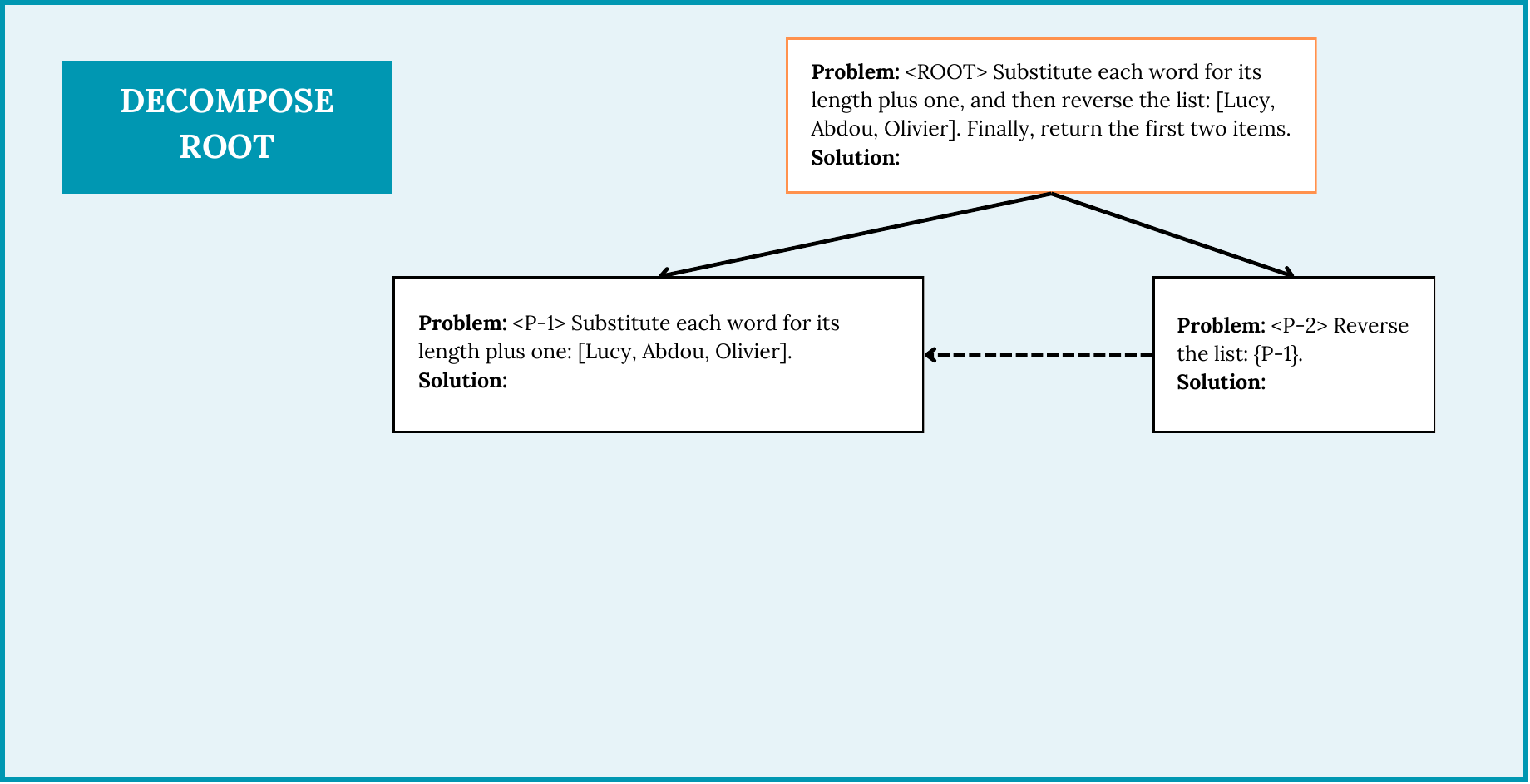}
\end{figure}

\begin{figure}[H]
    \centering
    \includegraphics[width=0.97\linewidth,keepaspectratio]{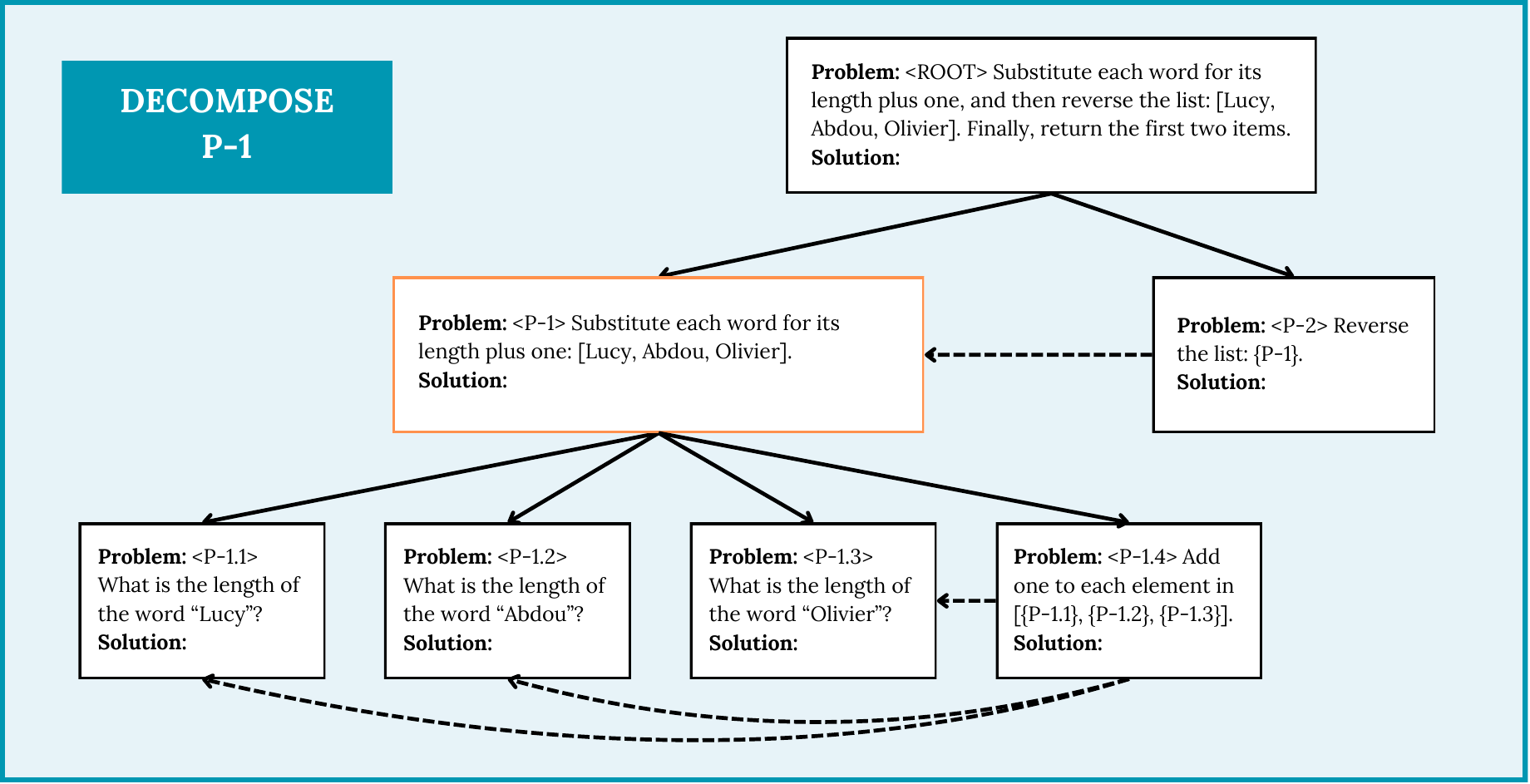}
\end{figure}

\begin{figure}[H]
    \centering
    \includegraphics[width=0.97\linewidth,keepaspectratio]{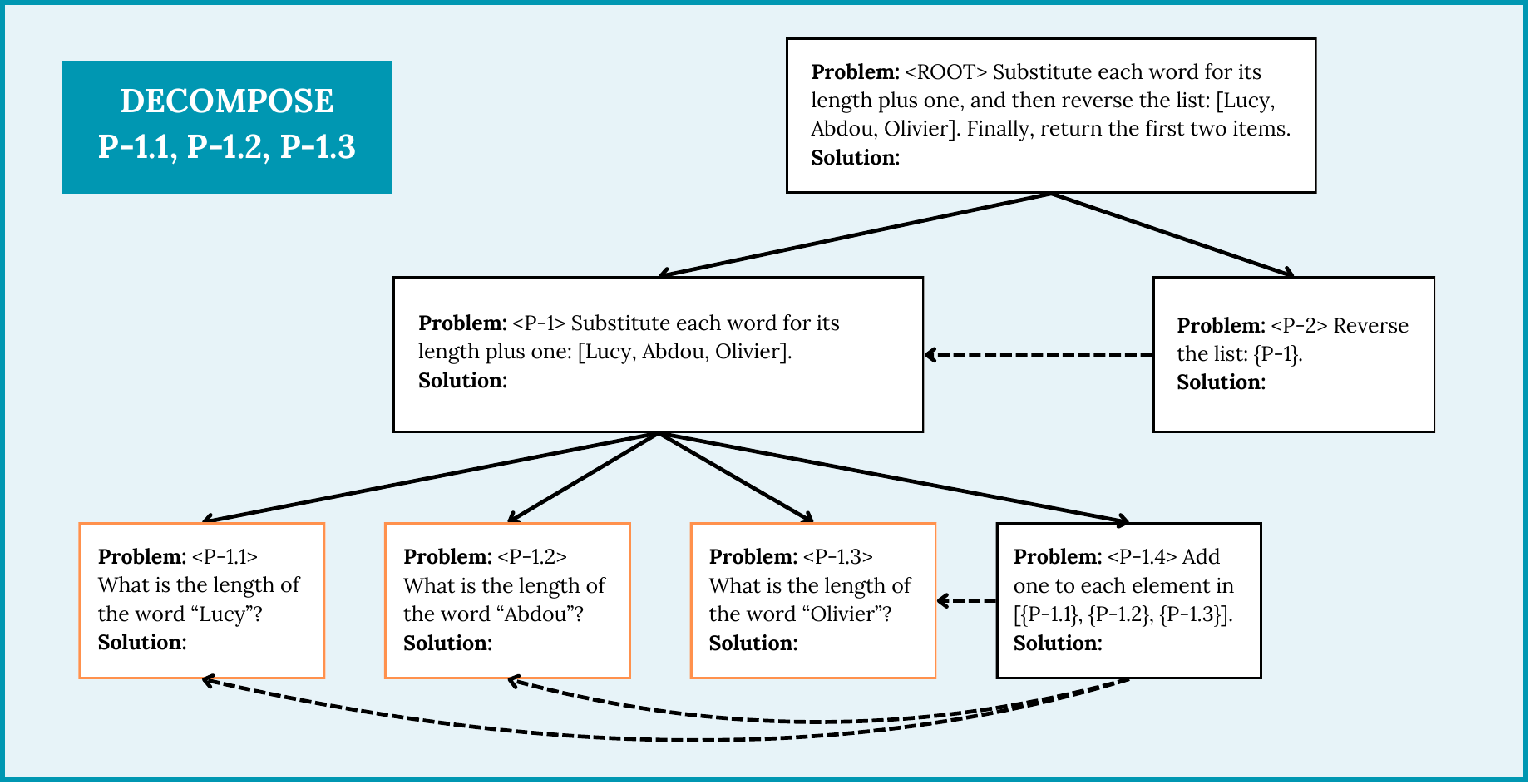}
\end{figure}

\begin{figure}[H]
    \centering
    \includegraphics[width=0.97\linewidth,keepaspectratio]{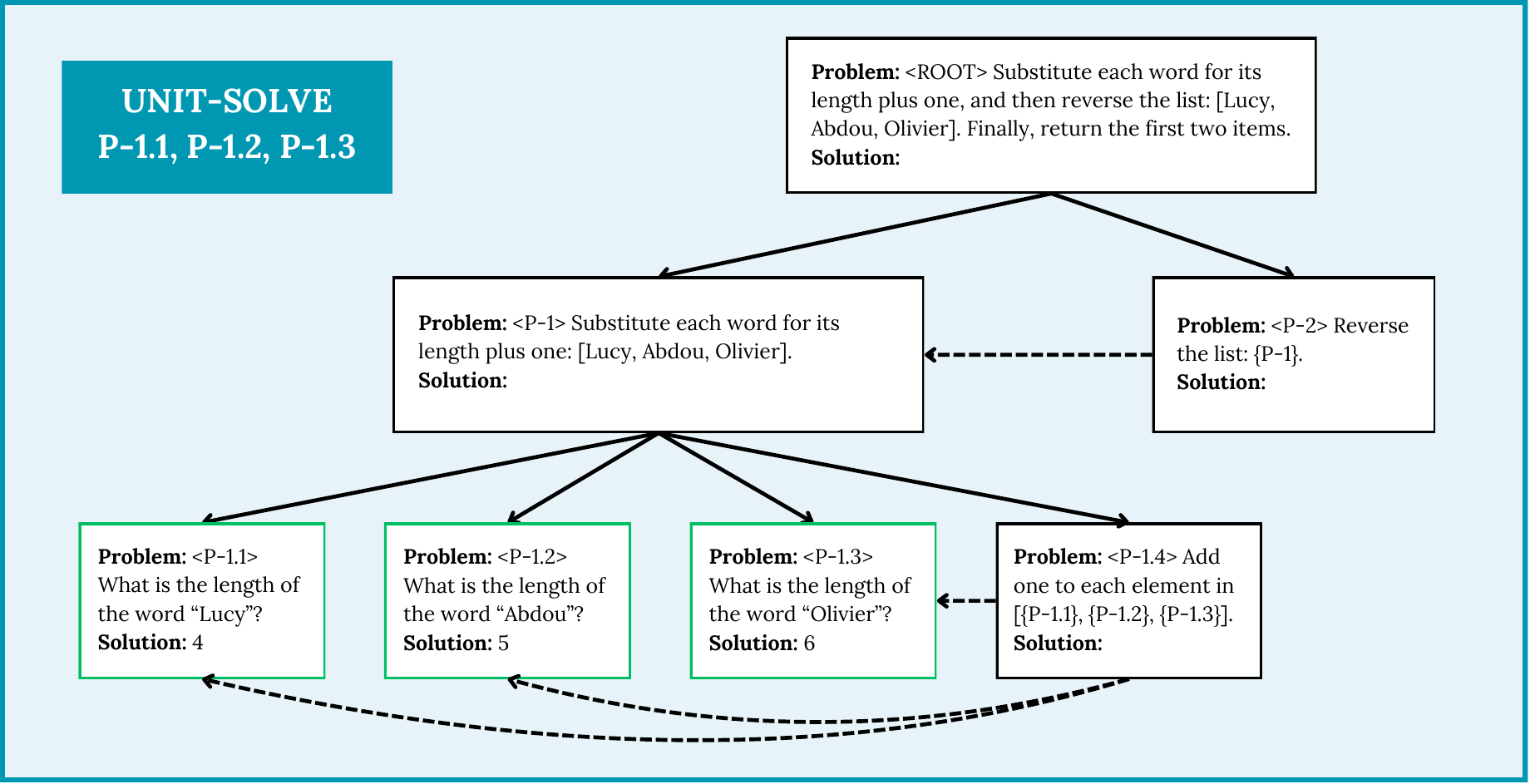}
\end{figure}

\begin{figure}[H]
    \centering
    \includegraphics[width=0.97\linewidth,keepaspectratio]{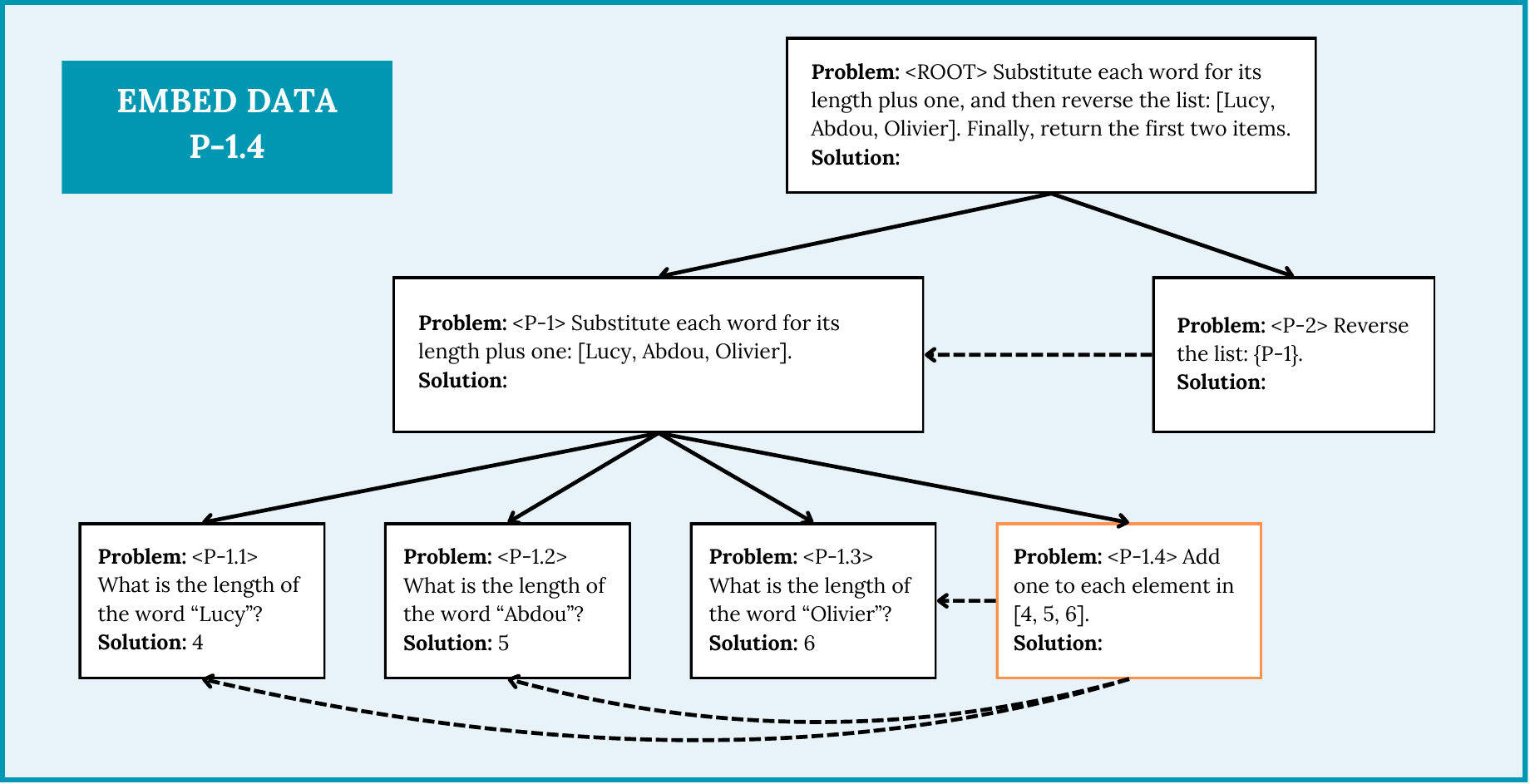}
\end{figure}

\begin{figure}[H]
    \centering
    \includegraphics[width=0.97\linewidth,keepaspectratio]{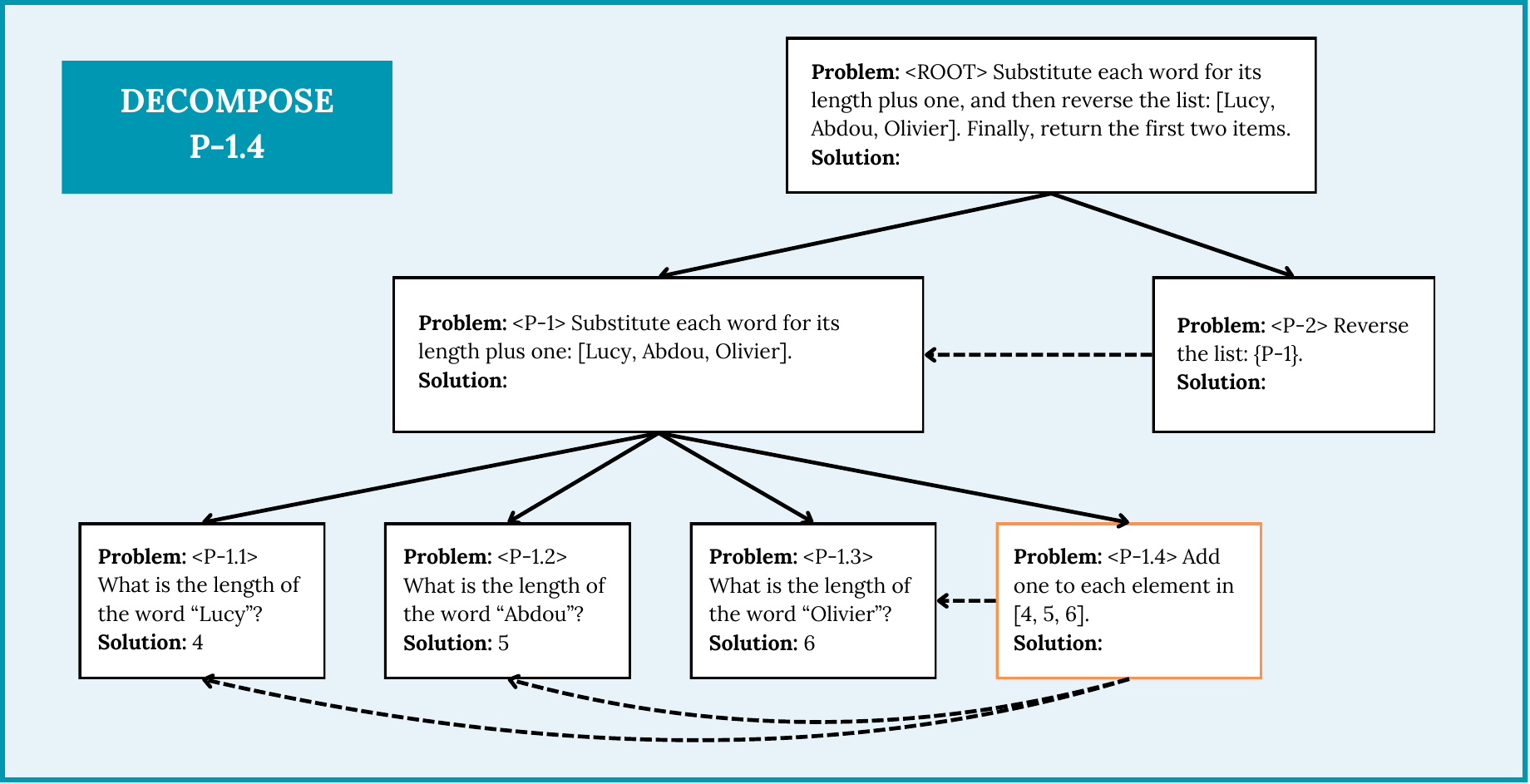}
\end{figure}

\begin{figure}[H]
    \centering
    \includegraphics[width=0.97\linewidth,keepaspectratio]{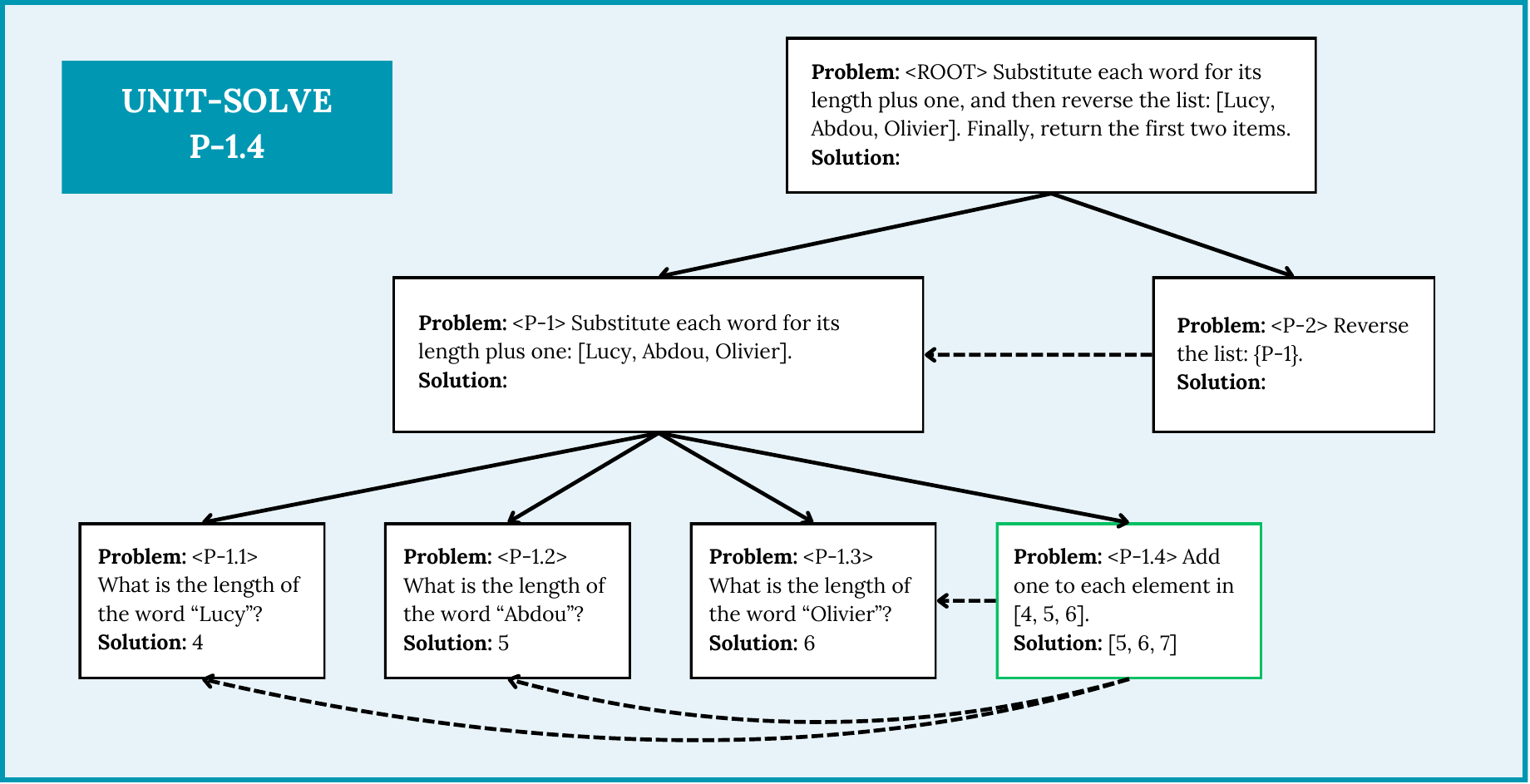}
\end{figure}

\begin{figure}[H]
    \centering
    \includegraphics[width=0.97\linewidth,keepaspectratio]{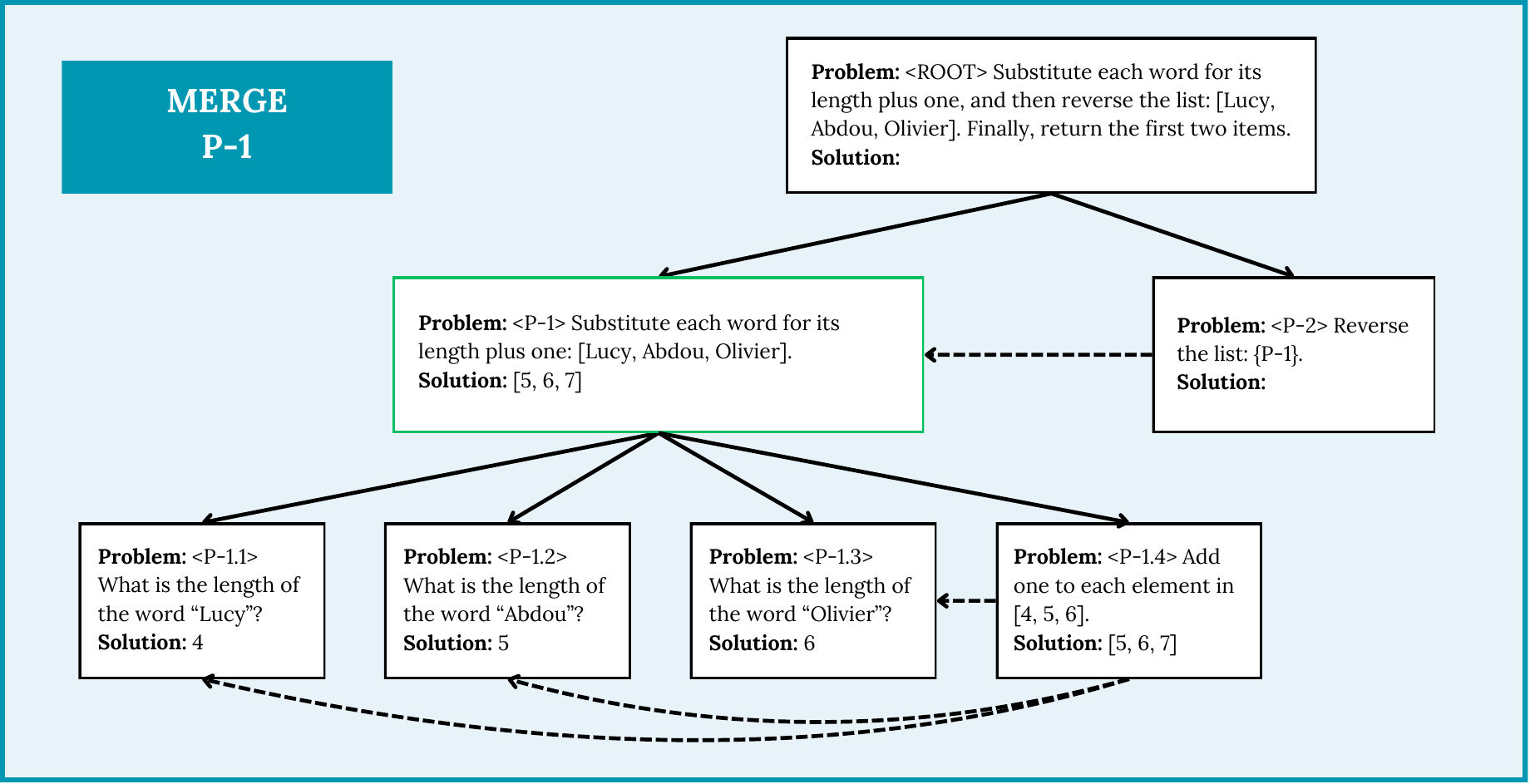}
\end{figure}

\begin{figure}[H]
    \centering
    \includegraphics[width=0.97\linewidth,keepaspectratio]{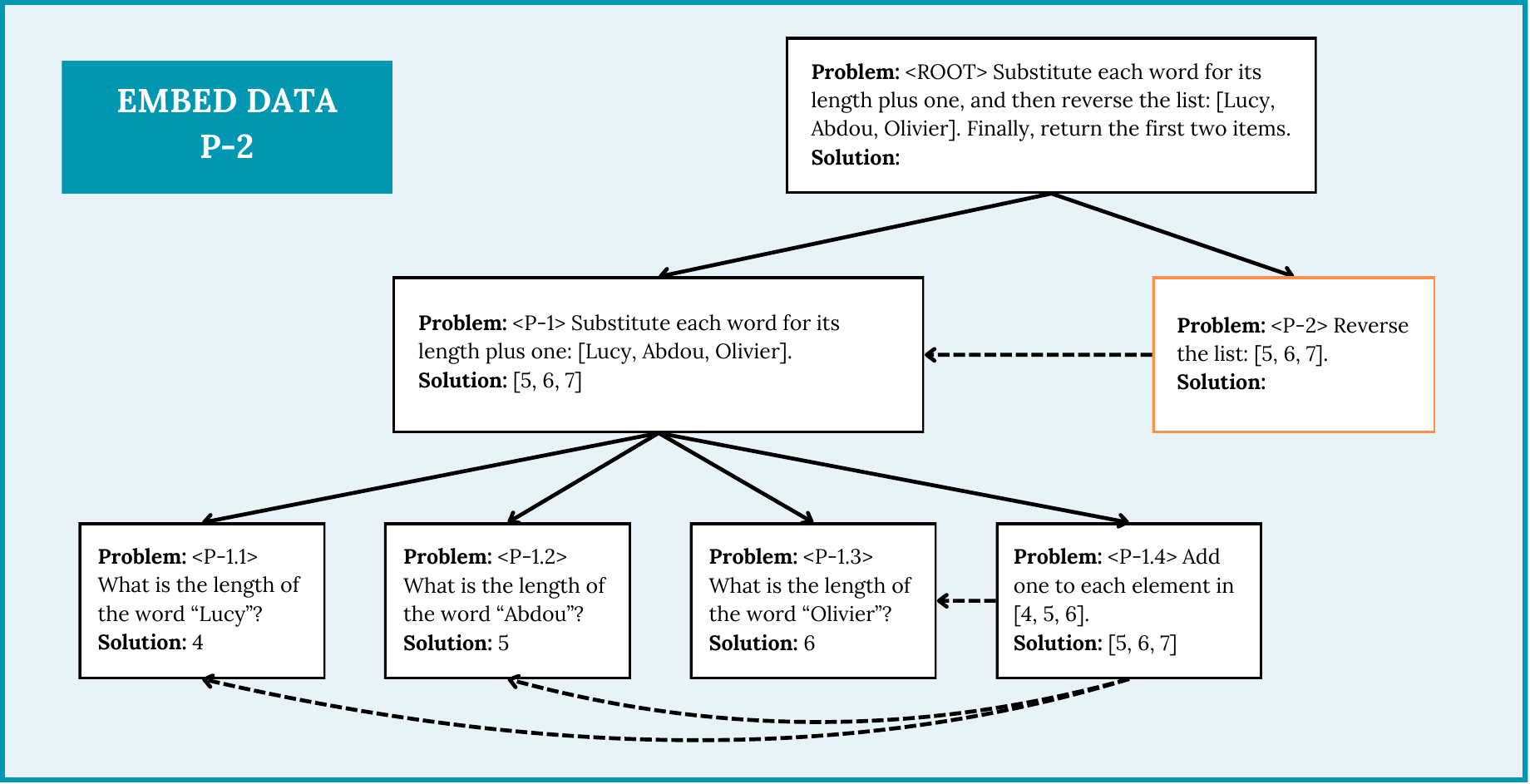}
\end{figure}

\begin{figure}[H]
    \centering
    \includegraphics[width=0.97\linewidth,keepaspectratio]{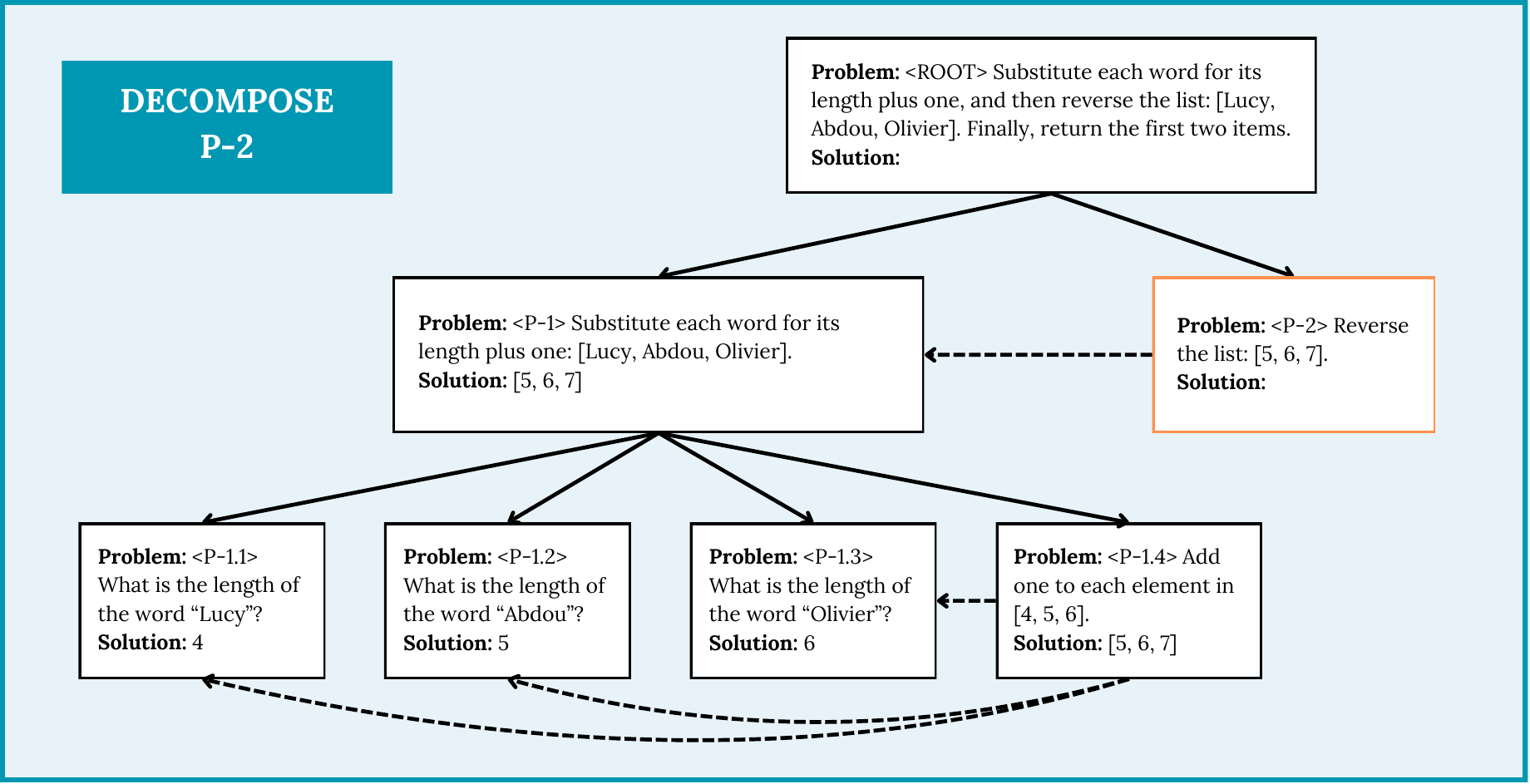}
\end{figure}

\begin{figure}[H]
    \centering
    \includegraphics[width=0.97\linewidth,keepaspectratio]{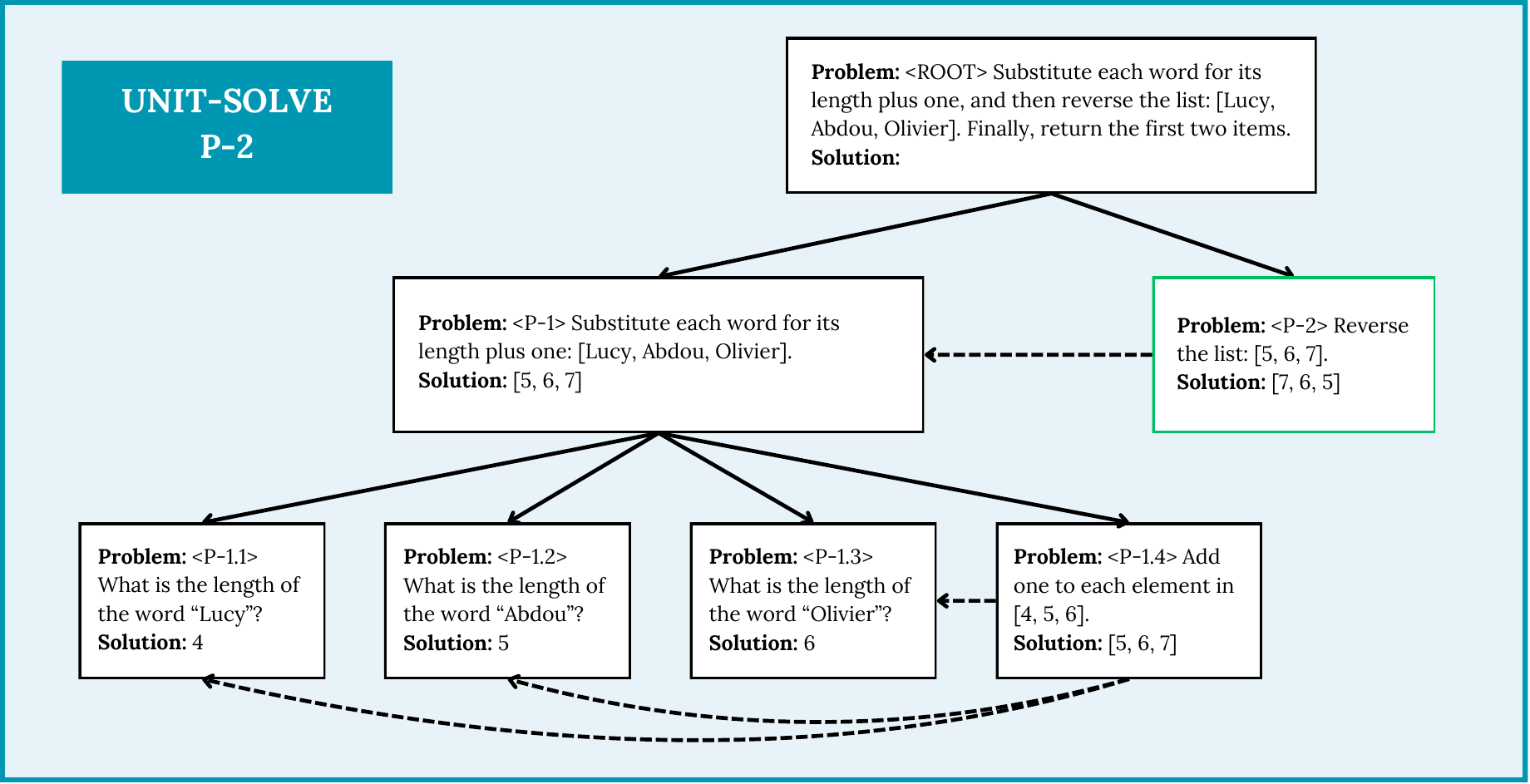}
\end{figure}

\begin{figure}[H]
    \centering
    \includegraphics[width=0.97\linewidth,keepaspectratio]{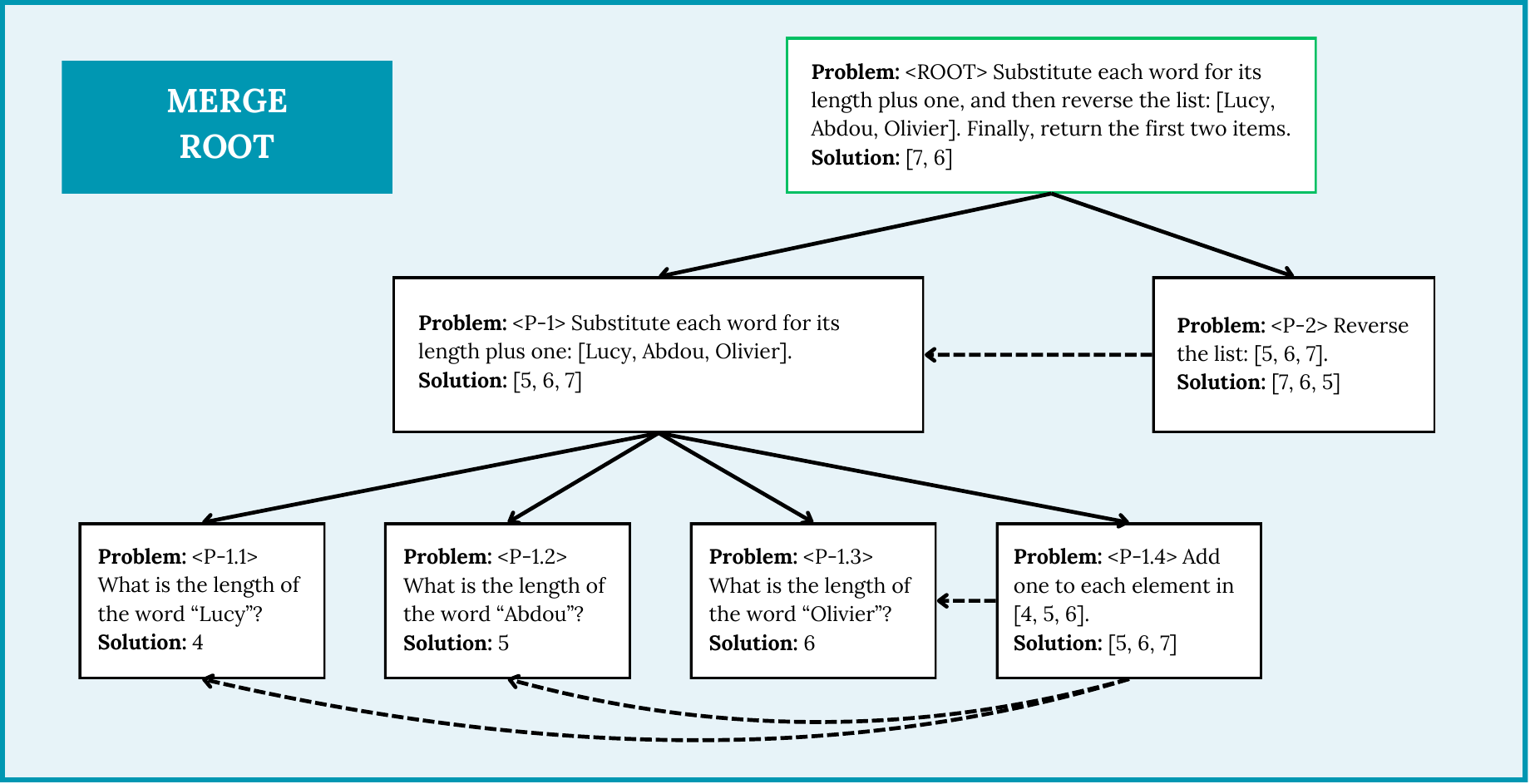}
\end{figure}

\end{document}